\documentclass[11pt]{article}

\usepackage{dsfont}
\usepackage[colorlinks, linkcolor=magenta, anchorcolor=cyan, citecolor=blue]{hyperref}

\usepackage{amsthm,amsfonts,amsmath,amssymb,epsfig,color,float,graphicx,verbatim, enumitem}
\usepackage{multirow}
\usepackage{algorithm}
\usepackage[noend]{algpseudocode}
\usepackage{enumitem}

\usepackage{enumitem}

\usepackage{framed}
\usepackage{nicefrac}
\usepackage[noabbrev,capitalise,nameinlink]{cleveref}
\usepackage{cleveref}

\newtheorem{theorem}{Theorem}[section]

\newtheorem{lemma}[theorem]{Lemma}
\newtheorem{informal theorem}[theorem]{Theorem (informal statement)}

\newtheorem{proposition}[theorem]{Proposition}

\newtheorem{claim}[theorem]{Claim}
\newtheorem{fact}[theorem]{Fact}

\newtheorem{remark}[theorem]{Remark}

\theoremstyle{definition}
\newtheorem{definition}[theorem]{Definition}
\newcommand{\eqdef}{\stackrel{{\mathrm {\footnotesize def}}}{=}}

\newcommand{\bx}{\mathbf{x}}
\newcommand{\by}{\mathbf{y}}
\newcommand{\bz}{\mathbf{z}}
\newcommand{\bu}{\mathbf{u}}
\newcommand{\bv}{\mathbf{v}}
\newcommand{\bw}{\mathbf{w}}

\newcommand{\bp}{\mathbf{p}}

\newcommand{\bI}{\mathbf{I}}
\newcommand{\bU}{\mathbf{U}}
\newcommand{\bV}{\mathbf{V}}

\newcommand{\B}{\mathbb{B}}

\newcommand{\vc}{\mathrm{VC}}
\newcommand{\A}{\mathcal{A}}
\renewcommand{\P}{\mathcal{P}}
\renewcommand{\B}{\mathcal{B}}
\newcommand{\C}{\mathcal{C}}
\renewcommand{\S}{\mathbf{S}}
\renewcommand{\L}{\mathcal{L}}
\newcommand{\T}{\mathbf{T}}
\newcommand{\U}{\mathcal{U}}
\newcommand{\V}{\mathcal{V}}

\newcommand{\I}{\mathbb{I}}
\newcommand{\p}{\mathbf{P}}

\newcommand{\R}{\mathbb{R}}

\newcommand{\Z}{\mathbb{Z}}
\newcommand{\N}{\mathbb{N}}
\newcommand{\E}{\mathbf{E}}
\newcommand{\eps}{\epsilon}
\newcommand{\dtv}{d_{\mathrm{TV}}}
\newcommand{\pr}{\mathbf{Pr}}
\renewcommand{\Pr}{\mathbf{Pr}}
\newcommand{\poly}{\mathrm{poly}}

\newcommand{\sgn}{\mathrm{sign}}
\newcommand{\sign}{\mathrm{sign}}

\newcommand{\D}{\mathcal{D}}

\newcommand{\bb}{\mathbf{b}}

\newcommand{\wt}{\widetilde}
\newcommand{\wh}{\widehat}

\author{
Ilias Diakonikolas\thanks{Supported by NSF Medium Award CCF-2107079,
NSF Award CCF-1652862 (CAREER), a Sloan Research Fellowship, and
a DARPA Learning with Less Labels (LwLL) grant.}\\
University of Wisconsin-Madison\\
{\tt ilias@cs.wisc.edu}\\
\and
Daniel M. Kane\thanks{Supported by NSF Medium Award CCF-2107547,
NSF Award CCF-1553288 (CAREER), a Sloan Research Fellowship, and a grant from CasperLabs.}\\
University of California, San Diego\\
{\tt dakane@cs.ucsd.edu}\\
\and
Yuxin Sun\thanks{Supported by NSF Award CCF-1652862 (CAREER) and NSF Medium Award CCF-2107547.}\\
University of Wisconsin-Madison\\
{\tt yxsun@cs.wisc.edu}\\
}

\title{SQ Lower Bounds for Learning Mixtures of Linear Classifiers}

\begin{document}

\maketitle

\begin{abstract}
We study the problem of learning mixtures of linear 
classifiers under Gaussian 
covariates. Given sample access to a mixture of 
$r$ distributions on $\R^n$ of the form $(\bx,y_{\ell})$, $\ell 
\in [r]$, where $\bx\sim\mathcal{N}(0,\mathbf{I}_n)$ and 
$y_\ell=\sgn(\langle\bv_\ell,\bx\rangle)$ for an unknown unit vector $\bv_\ell$, the goal is to learn the underlying 
distribution in total variation distance. Our main result is a 
Statistical Query (SQ) lower bound suggesting that known 
algorithms for this problem are essentially best 
possible, even for the special case of uniform mixtures.
In particular, we show that 
the complexity of any SQ algorithm for the problem is 
$n^{\poly(1/\Delta) \log(r)}$, where $\Delta$ is a lower bound on the pairwise $\ell_2$-separation between the $\bv_\ell$'s. The key technical 
ingredient underlying our result is a new construction of 
spherical designs that may be of independent 
interest. 
\end{abstract}

\setcounter{page}{0}

\thispagestyle{empty}

\newpage

\section{Introduction}\label{sec:intro}

The motivation behind this work is to understand the 
computational complexity of 
learning high-dimensional latent variable (aka mixture) models. 
The task of learning various mixture models has a long history 
in statistics with an early focus on sample efficiency, starting with the pioneering work of Karl Pearson~\cite{Pearson:94} on learning Gaussian mixtures.
During the past decades, an extensive line of work in machine learning and 
theoretical computer science has made significant progress on the 
computational aspects of this general question for a range of mixture models, including mixtures of high-dimensional Gaussians~\cite{Dasgupta:99, AroraKannan:01, VempalaWang:02, AchlioptasMcSherry:05, KSV08, BV:08, MoitraValiant:10, RV17-mixtures, HL18-sos, DKS18-list, DHKK20, LM20-gmm, BD+20-gmm, DKKLT21}, 
mixtures of linear regressions~\cite{sedghi2016provable, li2018learning, CLS20, DK20-ag}, 
and more generally mixtures of experts~\cite{jordan1994hierarchical, xu2009combining, masoudnia2014mixture,makkuva2019breaking, ho2022convergence}.

In this paper, we focus on mixtures of {\em linear classifiers}, a classical 
supervised probabilistic model that has been intensely investigated 
from both statistical and 
algorithmic standpoints~\cite{BlumC92, sun2014learning, li2017non, 
gandikota2020recovery, chen2022algorithms}. 
A linear classifier (or 
halfspace) is any Boolean function $h: \R^n \to \{ \pm 1 \}$ of the form
$h(\bx) = \sgn(\langle\bv,\bx \rangle)$, where $\bv \in \R^n$ is known as the weight 
vector and the univariate function $\sgn$ is defined by $\sgn(t) = 1$ for 
$t \geq 0$ and $\sgn(t) = -1$ otherwise. 
For an integer $r \geq 2$, we can now formally describe an 
$r$-mixture of linear classifiers. 
The parameters of the model contain $r$ unknown positive weights 
$w_1,\ldots,w_r$ with $\sum_{\ell=1}^rw_\ell=1$,
and $r$ unknown unit vectors $\bv_1,\ldots,\bv_r \in \R^n$.
A random sample is drawn from the underlying distribution 
$D(\bx,y)$ as follows: 
the sample oracle selects the index $\ell\in[r]$ with probability $w_\ell$,
and we then receive a random point 
$(\bx,y) \in \R^{n} \times \{ \pm 1\} $ 
where $\bx\sim\mathcal{N}(0,\mathbf{I}_n)$ and 
$y_\ell=\sgn(\langle\bv_\ell,\bx\rangle)$.
The goal is to approximately estimate the  
model either by learning the underlying distribution $D(\bx,y)$ in total variation distance (density estimation), or by approximately recovering the hidden parameters, i.e., $w_{\ell}$ and $\bv_\ell$, $\ell\in[r]$ (parameter estimation).
An algorithm for parameter estimation can be used
for density estimation (because closeness in parameters can be shown 
to imply closeness in total variation distance). In that sense,
parameter estimation is a harder problem. 

Before we proceed to describe prior algorithmic work, we provide
basic background on the sample complexity of the problem.
We start by noting that density estimation for $r$-mixtures of linear 
classifiers on $\R^n$ is information-theoretically solvable using 
$\poly(n, r)$ samples (with the optimal bound being $\tilde{\Theta}
(n r/\eps^2)$ for total variation error $\eps$) without any 
assumptions on the components. 
In contrast, for parameter 
estimation to be  information-theoretically solvable with polynomial sample complexity, 
some further assumptions are needed. 
The typical assumption involves some kind of pairwise separation
between the component vectors and a lower bound on the mixing weights. 
Let $\Delta \eqdef \min_{i \neq j}\{\|\bv_i - \bv_j\|_2, \|\bv_i+\bv_j\|_2\} >0$ be the 
pairwise separation between the components 
and $w_{\min}$ be the minimum mixing weight.  
Under these assumptions, the parameter estimation problem 
is solvable using $\poly(n, r, 1/\Delta, 1/w_{\min}, 1/\eps)$ samples 
to achieve parameter error of $\eps$. In both cases, the term ``information-theoretically 
solvable'' is used to mean that a {\em sample-efficient} algorithm exists, without any 
constraints on computational efficiency. The main question addressed in this paper is 
to what extent a {\em computationally efficient} learning algorithm exists.

On the algorithmic front, \cite{chen2022algorithms} gave 
parameter estimation algorithms with provable guarantees 
under a $\Delta$-separation assumption. 
Specifically, \cite{chen2022algorithms} provided two 
algorithms with different complexity guarantees. 
Their first algorithm has sample 
and computational complexity 
$\poly(n^{O(\log(r) / \Delta^2)}, 1/w_{\min}, 1/\eps)$, 
while their second algorithm has complexity $\poly((n/\Delta)^r, 1/w_{\min}, 1/\eps)$. 
Here we focus on settings where the number of components $r$
is large and cannot be viewed as a constant. 
For the sake of intuition, it is instructive to simplify these upper bounds 
for the regime of uniform mixtures 
(corresponding to the special case that $w_\ell = 1/r$ for all $\ell\in[r]$) 
and $\eps$ is not too small. For this regime, the complexity upper bound achieved in \cite{chen2022algorithms} is 
$\min \{ n^{O(\log(r) / \Delta^2)}, (n/\Delta)^{O(r)} \}$. 
Concretely, if the separation $\Delta$ is $\Omega(1)$ or even $1/\poly\log(r)$, the first 
term yields a quasi-polynomial upper bound. On the other hand, for $\Delta = O(1/r^c)$, for 
a constant $c>0$, the resulting upper bound is $n^{\poly(r)}$. 
In both regimes, we observe a 
{\em super-polynomial} gap between the information-theoretic sample complexity --- 
which is $\poly(n, r, \Delta)$ --- 
and the sample complexity of the \cite{chen2022algorithms} algorithms. 
It is thus natural to ask if this gap is inherent. 
\begin{center}
{\em What is the complexity of learning mixtures of linear classifiers? \\
Is there an algorithm with significantly better {\em sample-time tradeoff}?}
\end{center}
We study these questions in a well-studied restricted model of computation,
known as the Statistical Query (SQ) model~\cite{Kearns:98} (and, via~\cite{brennan2020statistical}, also for low-degree polynomial tests).
{\em Our main result is that in both of these models 
the complexity of the above-mentioned algorithms is essentially best possible.} 
Along the way, we establish new results on the existence of 
spherical designs that may be of independent interest.


\subsection{Our Results} \label{ssec:results}

In order to formally state our lower bounds, we require basic background on SQ algorithms.

\paragraph{Basics on SQ Model} 
SQ algorithms are a class of algorithms
that, instead of having direct access to samples, 
are allowed to query expectations of bounded functions of the distribution. 
Formally, an SQ algorithm has access to the following standard oracle.
\begin{definition}[$\mathrm{STAT}$ Oracle]\label{def:stat}
Let $D$ be a distribution on $\R^n$. A \emph{Statistical Query} 
is a bounded function $f:\R^n\to[-1,1]$. For $\tau>0$, the $\mathrm{STAT}(\tau)$ 
oracle responds to the query $f$ with a value $v$ such that $|v-\E_{\bx\sim D}[f(\bx)]|\le\tau$. 
A \emph{Statistical Query (SQ) algorithm} is an algorithm 
whose objective is to learn 
an unknown distribution $D$ by making adaptive calls to the $\mathrm{STAT}(\tau)$ oracle.
\end{definition}

\noindent The SQ model was introduced in~\cite{Kearns:98}. 
Subsequently, the model has been extensively studied 
in a range of contexts~\cite{Feldman16b}).
The class of SQ algorithms is broad and captures a range of known 
supervised learning algorithms. 
More broadly, several known algorithmic techniques in machine learning
are known to be implementable using SQs~\cite{FeldmanGRVX17, FeldmanGV17}.

\medskip

Our main result is a near-optimal SQ lower bound for mixtures of 
linear classifiers that applies even for the uniform case.
Specifically, we establish the following:

\begin{theorem}[Main Result:  SQ Lower Bound for Uniform Mixtures]\label{thm:main2}
Let $\epsilon\le c\Delta/r$ for some universal constant $c>0$ sufficiently small.
Then, any SQ algorithm that learns a uniform mixture of linear classifiers with directions $\bv_1,\ldots,\bv_r\in\mathbb{S}^{n-1}$ satisfying 
$\min_{i \neq j}\{\|\bv_i-\bv_j\|_2,\|\bv_i+\bv_j\|_2\}\ge\Omega(\Delta)$ for some $r^{-1/10}\le\Delta<1$,
within error $\epsilon$ in total variation distance must either use queries of tolerance $n^{-\poly(1/\Delta)\log r}$,
or make at least $2^{n^{\Omega(1)}}$ queries.
\end{theorem}

Informally speaking,
~\cref{thm:main2} shows that no SQ algorithm 
can perform density estimation for uniform mixtures of linear classifiers 
to small accuracy with a sub-exponential in $n^{\Omega(1)}$ many queries, 
unless using queries of very small tolerance  -- that would require at least
$n^{\poly(1/\Delta)\log r}$ samples to simulate.
This result can be viewed as a near-optimal information-computation tradeoff
for the problem, within the class of SQ algorithms. In more detail, for 
$\Delta = \Omega(1)$, we obtain a quasi-polynomial SQ lower bound of $n^{\Omega(\log r)}$;
while for $\Delta = 1/r^{c}$, for some constant $0< c<1/2$, we obtain an SQ lower bound
of $n^{\poly(r)}$. In both cases, our SQ lower bounds qualitatively match
the previously known algorithmic guarantees~\cite{chen2022algorithms} 
(that are efficiently implementable in the SQ model).

A conceptual implication of~\cref{thm:main2} is that 
the uniform (i.e., equal mixing weights) case is essentially as hard as the general
case for density estimation of these mixtures. In contrast,  
for related mixture models, specifically for mixtures of 
Gaussians, there is recent evidence that restricting the weights may make the problem computationally easier~\cite{BS21}.

\begin{remark} \label{rem:Delta}
{\em We note that the condition $\Delta\ge r^{-c}$, for some constant $0<c<1$, is necessary in the statement of~\cref{thm:main2} for the following reason: the algorithmic result of~\cite{chen2022algorithms} has sample and computational complexity $\min \{ n^{O(\log(r) / \Delta^2)}, (n/\Delta)^{O(r)} \}$,
which will be $(n/\Delta)^{O(r)}\ll n^{\poly(1/\Delta)\log r}$ if $\Delta$ is a sufficiently small inverse polynomial in $r$.
}
\end{remark}


\begin{remark} \label{rem:ld-test}
{\em Our SQ lower bound result has immediate implications
to another well-studied restricted computational model --- that of 
low-degree polynomial tests~\cite{HopkinsS17,HopkinsKPRSS17,Hopkins-thesis}. 
\cite{brennan2020statistical} established that (under certain assumptions) 
an SQ lower bound also implies a qualitatively similar lower bound in the low-degree model. 
This connection can be used as a black-box to show a similar lower bound for low-degree 
polynomials.}
\end{remark}

\noindent The key technical ingredient required for our SQ lower bound 
is a theorem establishing the existence of spherical designs with appropriate
properties. The definition of a spherical design follows.

\begin{definition}[Spherical Design]\label{def:sph-design}
Let $t$ be an odd integer.
A set of points $\bx_1,\ldots,\bx_r\in\mathbb{S}^{n-1}$ is called a spherical $t$-design if
$\E[p(\bx)]=\frac{1}{r}\sum_{i=1}^rp(\bx_i)$
holds for every homogeneous $n$-variate polynomial $p$ 
of degree $t$, where the expectation is taken over the 
uniform distribution on the unit sphere $\mathbb{S}^{n-1}$.
\end{definition}

We note that this definition differs slightly 
from the usual definition of spherical design,
which requires that the equation in~\cref{def:sph-design} 
holds for all polynomials $p$ of degree {\em at most} $t$.
However, by multiplying by powers of $\|\bx\|^2_2$,
we note that our definition implies that the equation holds 
for every odd polynomial of degree at most $t$, 
and it is sufficient for the requirement of establishing our SQ lower bound.

Spherical designs have been extensively studied in combinatorial design theory 
and a number of efficient constructions are known, see, 
e.g.,~\cite{bannai1979tight,delsarte1991spherical,graf2011computation,bondarenko2013optimal,kane2015small,womersley2018efficient}. However, none of the known constructions
seem to be compatible with our separation assumptions.

We establish the following result that may be of independent
interest in this branch of mathematics.

\begin{theorem}[Efficient Spherical Design]\label{thm:sph-design-intro}
Let $t$ be an odd integer and $r\ge\binom{n+2t-1}{n-1}^5$. Let $\by_1,\ldots,\by_r$ be uniform random vectors over $\mathbb{S}^{n-1}$.
Then, with probability at least 99/100, there exist unit vectors $\bz_i\in\mathbb{S}^{n-1}$ very close to $\by_i$ such that $(\bz_1,\ldots,\bz_r)$ form a spherical $t$-design.
\end{theorem}

We mention here that the optimal sample complexity for the existence of spherical $t$-design (for even $t$) is $r = \Theta\big(\binom{n+t}{n}\big)$ and our result matches this within a polynomial factor.
\cref{thm:sph-design-intro} is essential for our construction of moment-matching. Roughly speaking, the resulting mixture matches moments in the sense we require if and only if the weight vectors $(\bv_1,\ldots,\bv_r)$ form a {\em spherical design}.
since this leads to the separation of the hidden weight vectors in the mixture of linear classifiers.
In order to guarantee the pairwise separation of the weight vectors $\bv_i$'s,
we note that with high probability the $\by_i$'s are pairwise separated and therefore if $\bz_i$ is sufficiently close to $\by_i$ they will be too.


\subsection{Technical Overview} \label{ssec:techniques}


Our starting point is the SQ lower bound technique of~\cite{DKS17-sq} and its generalization in~\cite{DiakonikolasKPZ21}.
In particular, our overall strategy is to find a mixture of homogeneous halfspaces 
in some ``small'' number $m$ of dimensions that matches its first $k$ moments 
with the distribution on $(\bx, y)$, where $y$ is independent of $\bx$.
By appropriately embedding this low-dimensional construction into $\R^n$ 
via a random projection of $\R^n \to \R^m$,
the SQ lower bound machinery of~\cite{DiakonikolasKPZ21} can be shown to imply 
that learning the projection requires either $2^{n^{\Omega(1)}}$ queries 
or some query of accuracy $n^{-\Omega(k)}$.
This generic step reduces our problem to finding an appropriate $m$-dimensional construction.
Such a construction is the main technical contribution of this work.

Note that specifying a mixture of halfspaces 
is equivalent to specifying a probability distribution 
with support consisting of at most $r$ unit vectors and
specifying the orthogonal vectors of the halfspaces.
It is not hard to see that the resulting mixture matches moments 
in the sense we require
if and only if these vectors form a {\em spherical $k$-design}~\cite{delsarte1991spherical}: 
namely, that for any odd polynomial $p$ of degree less than $k$
that the average value of $p$ over our distribution is the same 
as the average value of $p$ over the unit sphere (namely, equal to zero).
The bulk of our problem now reduces to finding constructions 
of weighted designs, where:
(1) The support size of the design is relatively small.
(2) The points in the design are pairwise separated.
(3) For equally weighted mixtures, we require that the weight of each vector in the design is uniform.

In $m=2$ two dimensions, there is a relatively simple explicit construction that can be given.
In particular, if we take $k$ evenly spaced points over the unit circle for some odd $k$,
this matches the first $k-1$ moments and has separation $\Omega(1/k)$.
(This is similar to 
an SQ construction in~\cite{DiakonikolasKKZ20} for a different setting.)
Unfortunately, we cannot obtain better separation in two dimensions, 
since any $k$ unit vectors in two dimensions will necessarily have 
some pair separated by at most $O(1/k)$.
Therefore, if we want constructions with greater separation, 
we will need to pick larger values of $m$.
Indeed, we show (\cref{prop:low-g}) that a {\em random} collection of approximately $m^{O(k)}$ points 
can form the support of an appropriate design.
This can be proved by applying linear programming duality.
Unfortunately, this argument does not allow us to control the mixing weights. Specifically, it may be the case that the minimum weight is exponentially 
small which seems unnatural in practice.

The case of equal weights requires a significantly more sophisticated argument.
In this case, we need to find a spherical design with uniform weight.
Indeed, merely selecting a random support will no longer work.
Although there is an extensive literature on finding efficient spherical 
designs~\cite{bannai1979tight,delsarte1991spherical,graf2011computation,bondarenko2013optimal,kane2015small,womersley2018efficient}, none of the known constructions
seem to be compatible with our separation assumptions.
In particular,
\cite{bondarenko2013optimal} proves that for each $r\ge c_mk^m$, there exists a spherical $k$-design, where $c_m$ is a constant depending only on the dimension $m$.
Although it is plausible that their construction can 
be adapted to satisfy our separation requirement, 
their sample complexity has a potentially bad dependence 
on the dimension $m$. On the other hand, the work of~\cite{kane2015small} 
achieves the optimal sample complexity $r$ up to polynomials 
for all $m$ and $k$. Unfortunately, the construction in~\cite{kane2015small} 
definitely does not satisfy our separation requirement. 
We note that the sample complexity of $r = \Theta\big(\binom{m+k}{k}\big)$ should be optimal, and our results are polynomial in this bound.

In conclusion, to establish our SQ lower bound, 
we need to prove a new technical result. 
Our starting point to that end will be 
the work of~\cite{bondarenko2013optimal}.
The basic plan of our approach will be to select random unit vectors $\by_1,\by_2,\ldots,\by_r\in\R^n$, and then --- 
applying topological techniques --- 
to show that there is some small perturbation 
set $\bx_1,\bx_2,\ldots,\bx_r\in\R^n$ that is a spherical design.
In particular, we will find a continuous function $F$ 
mapping degree-$k$ odd polynomials 
to sets of $r$ unit vectors such that, 
for all unit-norm polynomials $p$, 
if $F(p) = (\bz_1,\ldots,\bz_r)$ 
then $\sum_{i=1}^r p(\bz_i) > 0$.
Given this statement, 
a standard fixed point theorem~\cite{cho2006topological} will imply 
that our design 
can be found as $F(q)$ for some $q$ with norm less than one.
To construct the mapping $F$, we start with the points $\by_1,\ldots,\by_r$ 
and perturb each $\by_i$ in the direction of $\nabla_o p(\by_i)$
in order to try to increase the average value of $p$,
where $\nabla_o p(\by_i)$ is the component of $\nabla p(\by_i)$ orthogonal 
to the direction $\by_i$.
Intuitively, this construction should work because with high probability 
the average value of $p(\by_i)$ is already small 
(since the empirical average should approximate the true average),
the average gradient of $p(\by_i)$ is not too small, 
and the contributions to $p(\bz_i)$ coming from higher-order terms 
will not be large as long as $\bz_i$ is sufficiently close to $\by_i$.
These facts can all be made precise with some careful 
analysis of spherical harmonics (\cref{lem:increment}).

Finally, we need to show that the hardness of learning the appropriate projection 
implies hardness of learning the mixture.
By standard facts, it suffices to show that mixtures of linear classifiers 
are far from the distribution $(\bx, y)$, 
where the uniform label $y$ is independent of $\bx$, in total variation distance.
To show this, we prove that the total variation distance between 
any projection and the distribution where $y$ is independent of $\bx$ 
is at least $\Omega(\Delta/r)$ (\cref{lem:TV-lower-bound}).

\section{Preliminaries}\label{sec:prelims}

\paragraph{Notation} 
For $n \in \Z_+$, we denote $[n] \eqdef \{1,\ldots,n\}$.
For two distributions $p,q$ over a probability space $\Omega$,
let $\dtv(p,q)=\sup_{S\subseteq\Omega}|p(S)-q(S)|$
denote the total variation distance between $p$ and $q$.
In this article, we typically use small letters to denote random variables and vectors.
For a real random variable $x$, we use $\E[x]$ to denote the expectation.
We use $\Pr[\mathcal{E}]$ and $\mathbb{I}[\mathcal{E}]$ for
the probability and the indicator of event $\mathcal{E}$.
Let $\mathcal{N}_n$ denote the standard $n$-dimensional Gaussian distribution
and $\mathcal{N}$ denote the standard univariate Gaussian distribution.
Let $\mathbb{S}^{n-1}=\{\bx\in\R^n:\|\bx\|_2=1\}$ denote the $n$-dimensional unit sphere.
For a subset $S\subseteq\R^n$,
we will use $\mathcal{U}(S)$ to denote the uniform distribution over $S$.
We will use small boldface letters for vectors and capital boldface letters for matrices.
Let $\|\bx\|_2$ be the $\ell^2$-norm of the vector $\bx\in\R^n$.
For vectors $\bu,\bv\in\R^n$, we use $\langle\bu,\bv\rangle$ to denote their inner product.
We denote by $\L^2(\R^n,\mathcal{N}_n)$ the function space of all functions $f:\R^n\to\R$ such that $\E_{\bz\in\mathcal{N}_n}[f^2(\bz)]<\infty$.
The usual inner product for this space is $\E_{\bz\in\mathcal{N}_n}[f(\bz)g(\bz)]$.

For a matrix $\p\in\R^{m\times n}$,
we denote $\|\p\|_2,\|\p\|_F$ to be its spectral norm and Frobenius norm respectively.
For a tensor $A\in(\R^n)^{\otimes k}$,
let $\|A\|_2$ denote its spectral norm and $\|A\|_F$ denote its Frobenius norm.
We will use $A_{i_1,\ldots,i_k}$ to denote the coordinate of the $k$-tensor $A$ indexed by the $k$-tuple $(i_1,\ldots,i_k)$.
The inner product between $k$-tensors is defined by thinking 
of tensors as vectors with $n^k$ coordinates.

\medskip

We use the framework of Statistical Query (SQ) algorithms for problems 
over distributions~\cite{FeldmanGRVX17} and 
require the following standard definition.

\begin{definition}[Decision/Testing Problem over Distributions]\label{def:decision}
Let $D$ be a distribution and $\D$ be a family of distributions over $\R^n$. 
We denote by $\mathcal{B}(\mathcal{D},D)$ the decision (or hypothesis testing) 
problem in which the input distribution $D'$ is promised to satisfy either 
(a) $D'=D$ or (b) $D'\in\mathcal{D}$, and the goal of the algorithm 
is to distinguish between these two cases.
\end{definition}

\paragraph{Probabilistic Facts}  

We recall the definition of VC dimension for a 
set system.

\begin{definition}[VC-Dimension]\label{def:VC-dim}
For a class $\C$ of boolean functions on a set $\mathcal{X}$, the~\emph{VC-dimension} of $\C$ is the largest integer $d$ such that there exist $d$ points $x_1,\ldots,x_d\in\mathcal{X}$ such that for any boolean function $g:\{x_1,\ldots,x_d\}\to\{\pm1\}$, there exists an $f\in\C$ satisfying $f(x_i)=g(x_i),1\le i\le d$.
\end{definition}

\noindent We will use the following probabilistic inequality. 

\begin{lemma}[VC inequality, see, e.g.,~\cite{vershynin2018high}]\label{lem:VC-ineq}
Let $\C$ be a class of boolean functions on $\mathcal{X}$ with VC-dimension $d$, and let $X$ be a distribution on $\mathcal{X}$. Let $\epsilon>0$ and let $N$ be an integer at least a sufficiently large constant multiple of $d/\epsilon^2$. Then, if $X_1,X_2,\ldots,X_N$ are i.i.d. samples from $X$, we have that:
\begin{align*}
\pr\left[\sup_{f\in\C}\left|\frac{\sum_{j=1}^Nf(X_j)}{N}-\E[f(X)]\right|>\epsilon\right]=\exp(-\Omega(N\epsilon^2)).
\end{align*}
\end{lemma}

To apply the VC inequality in our context,
we additionally need the following fact which gives the VC-dimension 
of the class of bounded-degree polynomial threshold functions (PTFs).

\begin{fact}\label{fact:VC-dim}
Let $\C_k$ denote the class of degree-$k$ polynomial threshold functions (PTFs) on $\R^m$,
namely the collection of functions of the form $f(\bx)=\sgn(p(\bx))$ for some degree at most $k$ real polynomial $p$.
Then, the VC-dimension of $\C_k$ is $\vc(\C_k)=\binom{m+k}{k}$.
\end{fact}

\section{Warmup: SQ Lower Bounds for General Weight Mixtures}\label{sec:SQ-mixture-LTFs}

\paragraph{SQ Lower Bound Machinery} 
We start by defining the family of distributions that we will use to prove 
our SQ hardness result.
\begin{definition}[\cite{DiakonikolasKPZ21}] \label{def:low-g}
Given a function $g:\R^m\to[-1,+1]$,
we define $\D_g$ to be the class of distributions over $\R^n\times\{\pm1\}$ of the form $(\bx,y)$ such that
$\bx\sim\mathcal{N}_n$ and $\E[y\mid \bx=\bz]=g(\bU\bz)$,
where $\bU\in\R^{m\times n}$ with $\bU\bU^\intercal=\bI_m$.
\end{definition}

The following proposition states that if $g$ has zero low-degree moments,
then distinguishing $\D_g$ from the distribution $(\bx,y)$ with $\bx\sim\mathcal{N}_n,y\sim\mathcal{U}(\{\pm1\})$ is hard in the SQ model.
\begin{proposition}[\cite{DiakonikolasKPZ21}] \label{prop:low-g}
Let $g:\R^m\to[-1,1]$ be such that $\E_{\bx\sim\mathcal{N}_m}[g(\bx)p(\bx)]=0$,
for every polynomial $p:\R^m\to\R$ of degree less than $k$,
and $\D_g$ be the class of distributions from~\cref{def:low-g}.
Then, if $m\le n^a$, for some constant $a<1/2$,
any SQ algorithm that solves the decision problem $\mathcal{B}(\D_g,\mathcal{N}_n\times\mathcal{U}(\{\pm1\}))$ must either use queries of tolerance $n^{-\Omega(k)}$,
or make at least $2^{n^{\Omega(1)}}$ queries.
\end{proposition}

We will apply~\cref{prop:low-g} to establish our SQ lower bound for learning mixtures of linear classifiers.
The main technical contribution required to achieve this is the construction
of a class of distributions $\D_g$,
where each element in $\D_g$ represents a distribution of mixture of linear classifiers.
In particular, we will carefully choose some appropriate unit vectors $\bv_1,\ldots,\bv_r\in\R^m$ and non-negative weights $w_1,\ldots,w_r$ with $\sum_{\ell=1}^rw_\ell=1$,
such that $\bv_\ell$ are pairwise separated by $\Delta>0$.

Let $g(\bz)=\sum_{\ell=1}^rw_\ell \sgn(\bv_\ell^{\intercal}\bz),\bz\in\R^m$.
For an arbitrary matrix $\bU\in\R^{m\times n}$ with $\bU\bU^{\intercal}=\bI_m$,
we denote by $D_{\bU}$ the instance of mixture of linear classifiers with weight vectors $\bU^{\intercal}\bv_1,\ldots,\bU^{\intercal}\bv_r$ and weights $w_1,\ldots,w_r$.
In this way, we have that
$$\E_{(\bx,y)\sim D_{\bU}}[y\mid\bx=\bz]=\sum_{\ell=1}^rw_\ell \sgn((\bU^{\intercal}\bv_\ell)^{\intercal}\mathbf{z})=\sum_{\ell=1}^rw_\ell \sgn(\bv_\ell^{\intercal}\bU\mathbf{z})=g(\bU\mathbf{z}),\bz\in\R^n.$$

\subsection{Low-degree Moment Matching}\label{ssec:low-degree}
The following proposition shows that there exist unit vectors $\bv_1,\ldots,\bv_r$ and non-negative weights $w_1,\ldots,w_r$ with $\sum_{\ell=1}^rw_\ell=1$
such that $\bv_\ell$ are pairwise separated 
by some parameter $\Delta>0$ and the low-degree moments of $g$ vanish.
Note that since $g$ is odd (as the $\sign$ function is odd),
we only require $\E_{\bz\in\mathcal{N}_m}[g(\bz)p(\bz)]=0$ for every odd polynomial $p$ of degree less than $k$.

\begin{proposition}\label{prop:moment-matching}
Let $m=\frac{c\log r}{\log(1/\Delta)}$, where $r^{-1/10}\le\Delta<1$ and $1.99\le c\le2$ is a universal constant.
Let $k$ be a positive integer such that $r\ge C\binom{m+k}{k}$ for some constant $C>0$ sufficiently large.
There exist vectors $\bv_1,\ldots,\bv_r$ over the unit sphere $\mathbb{S}^{m-1}$ and non-negative weights $w_1,\ldots,w_r$ with $\sum_{\ell=1}^rw_\ell=1$, such that
\begin{itemize}[leftmargin=*]
\item $\E_{\bz\sim\mathcal{N}_m}[g(\bz)p(\bz)]=0$ holds for every odd polynomial $p$ of degree less than $k$,
where $g(\bz)=\sum_{\ell=1}^rw_\ell \sgn(\bv_\ell^{\intercal}\bz)$.
\item $\|\bv_i+\bv_j\|_2,\|\bv_i-\bv_j\|_2\ge\Omega(\Delta), \forall1\le i<j\le r$.
\end{itemize}
\end{proposition}

The proof proceeds as follows:
We uniformly sample unit vectors $\bv_1,\ldots,\bv_r\in\mathbb{S}^{m-1}$.
We will prove by the probabilistic method that both statements 
in~\cref{prop:moment-matching} hold with high probability.
We begin by showing that 
$\pr[\min_{1\le i<j\le r}\{\|\bv_i-\bv_j\|_2,\|\bv_i+\bv_j\|_2\}<O(\Delta)]$ is small.
By symmetry, it suffices to bound $\pr[\min_{1\le i<j\le r}\|\bv_i-\bv_j\|_2<O(\Delta)]$.
Let $\Theta = \min_{1\le i<j\le r}\arccos\langle \bv_i,\bv_j\rangle$. 
We require the following facts:
\begin{fact}[Proposition 3.5 in~\cite{brauchart2018random}]\label{fact:arg}
$\pr\left[\Theta\ge\gamma r^{-\frac{2}{m-1}}\right]\ge1-\frac{\kappa_{m-1}}{2}\gamma^{m-1}$,
where $\kappa_m=\frac{\Gamma((m+1)/2)}{m\sqrt{\pi}\Gamma(m/2)}\in\left[\frac{1}{m}\sqrt{\frac{m-1}{2\pi}},\frac{1}{m}\sqrt{\frac{m+1}{2\pi}}\right]$.
\end{fact}
Applying~\cref{fact:arg} by taking $\gamma=\left(\frac{1}{50\kappa_{m-1}}\right)^{\frac{1}{m-1}}$
and $r=\frac{1}{\sqrt{50\kappa_{m-1}\Delta^{m-1}}}$ yields that $\pr[\Theta\ge\Delta]\ge99/100$.
Given some fixed random vectors $\bv_1,\ldots,\bv_r\in\mathbb{S}^{m-1}$,
we will prove that with high probability there exist non-negative weights $w_1,\ldots,w_r$ with $\sum_{\ell=1}^rw_\ell=1$ such that $\E_{\bz\sim\mathcal{N}_m}[g(\bz)p(\bz)]=0$ holds for every odd polynomial $p$ of degree less than $k$,
where $g(\bz)=\sum_{\ell=1}^rw_\ell \sgn(\bv_\ell^{\intercal}\bz)$.
Noting that $\E_{\bz\sim\mathcal{N}_m}\left[g(\bz)p(\bz)\right]=\sum_{\ell=1}^rw_\ell\E_{\bz\sim\mathcal{N}_m}\left[p(\bz)\sgn(\bv_\ell^{\intercal}\bz)\right]$,
it suffices to bound the probability that there exist non-negative weights $w_1,\ldots,w_r$ with $\sum_{\ell=1}^rw_\ell=1$
such that $\sum_{\ell=1}^rw_\ell\E_{\bz\sim\mathcal{N}_m}\left[p(\bz)\sgn(\bv_\ell^{\intercal}\bz)\right]=0$ holds for every odd polynomial $p$ of degree less than $k$.
Our main technical lemma is the following:

\begin{lemma}\label{lem:poly}
Let $p:\R^m\to\R$ be a polynomial of degree less than $k$ and $f\in\L^2(\R,\mathcal{N})$.
Let $\bv\in\mathbb{S}^{m-1}$ be a unit vector.
We have that $\E_{\bz\sim\mathcal{N}_m}\left[p(\bz)f(\bv^{\intercal}\bz)\right]$ is a polynomial in $\bv$ of degree less than $k$.
\end{lemma}
The proof proceeds by analyzing $p$ and $f$ as Hermite expansions.
For completeness, we defer the proof of~\cref{lem:poly} to~\cref{sec:omit-SQ-mixture-LTFs}.
From~\cref{lem:poly}, we can write $\sum_{\ell=1}^rw_\ell\E_{\bz\sim\mathcal{N}_m}\left[p(\bz)\sgn(\bv_\ell^{\intercal}\bz)\right]=\sum_{\ell=1}^r w_\ell q(\bv_\ell)$ for some odd polynomial $q$ of degree less than $k$ (since $\sgn$ is odd).
In this way, it suffices to show that with high probability there exist non-negative weights $w_1,\ldots,w_r$ with $\sum_{\ell=1}^rw_\ell=1$
such that $\sum_{\ell=1}^r w_\ell q(\bv_\ell)=0$ for all odd polynomials $q$ of degree less than $k$.
To prove this, we leverage the following lemma: 

\begin{lemma}\label{lem:equiv}
The following two statements are equivalent.
\begin{enumerate}[leftmargin=*]
\item There exist non-negative weights $w_1,\ldots,w_r$ with $\sum_{\ell=1}^rw_\ell=1$ such that $\sum_{\ell=1}^r w_\ell q(\bv_\ell)=0$ for all odd polynomials $q$ of degree less than $k$.
\item There does not exist any odd polynomial $q$ of degree less than $k$ such that $q(\bv_\ell)>0,1\le\ell\le r$.
\end{enumerate}
\end{lemma}

\noindent The proof of~\cref{lem:equiv} follows via 
a careful application of LP duality;
see~\cref{sec:omit-SQ-mixture-LTFs}.

\begin{proof}[Proof of~\cref{prop:moment-matching}]
Let $\Theta = \min_{1\le i<j\le r}\arccos\langle \bv_i,\bv_j\rangle$. 
Applying~\cref{fact:arg} by taking $\gamma=\left(\frac{1}{50\kappa_{m-1}}\right)^{\frac{1}{m-1}}$ yields $
\pr[\Theta\ge\Delta]\ge\pr\left[\Theta\ge\gamma r^{-\frac{2}{m-1}}\right]\ge99/100$
(since our choice of $m$ will imply that
$r=1/\sqrt{50\kappa_{m-1}\Delta^{m-1}}$).
This will imply that 
$$\pr\Big[\min_{1\le i<j\le r}\{\|\bv_i-\bv_j\|_2,\|\bv_i+\bv_j\|_2\}\ge\Omega(\Delta)\Big]\ge99/100 \;.$$
We then show that there exist non-negative weights $w_1,\ldots,w_r$ with $\sum_{\ell=1}^rw_\ell=1$
such that $\E_{\bz\sim\mathcal{N}_m}[g(\bz)p(\bz)]=0$ holds for every odd polynomial $p$ with degree less than $k$,
where $g(\bz)=\sum_{\ell=1}^rw_\ell \sgn(\bv_\ell^{\intercal}\bz)$.
By~\cref{lem:poly}, it suffices to show that
there exist non-negative weights $w_1,\ldots,w_r$ with $\sum_{\ell=1}^rw_\ell=1$ such that
$\sum_{\ell=1}^r w_\ell q(\bv_\ell)=0$ for all odd polynomials $q$ of degree less than $k$.
By~\cref{fact:VC-dim}, we have that $r\ge C \; \binom{m+k}{k}=C\vc(\C_k)$ for some sufficiently large constant $C>0$.
Note that for any odd polynomial $q$ in $m$ variables of degree less than $k$, we have that $\E[\sgn(q(\bv))]=0$. Therefore, by the VC-inequality (\cref{lem:VC-ineq}), we have that
\begin{align*}
&\quad\pr[\exists\text{ odd polynomial }q\text{ of degree less than }k\text{ such that }q(\bv_\ell)>0,1\le\ell\le r]
\\&\le\pr\left[\sup_{f\in\C_{k-1}}\frac{\sum_{\ell=1}^rf(\bv_\ell)}{r}=1\right]\le\pr\left[\sup_{f\in\C_k}\left|\frac{\sum_{\ell=1}^rf(\bv_\ell)}{r}-\E[f(\bv)]\right|\ge1\right]\\&\le\exp(-\Omega(r))\le1/100.
\end{align*}
Finally, applying~\cref{lem:equiv} completes our proof.
\end{proof}

\subsection{Proof of SQ Lower Bound for General Weights}
Let $k$ be a positive integer such that $r\ge C\binom{m+k}{k}$ for some constant $C>0$ sufficiently large.
Let $m=\frac{c\log r}{\log(1/\Delta)}$, where $r^{-1/10}\le\Delta<1$ and 
$1.99\le c\le2$. 
By~\cref{prop:moment-matching}, there exist unit vectors $\bv_1,\ldots,\bv_r\in\R^m$ such that
\begin{itemize}[leftmargin=*]
\item There exist non-negative weights $w_1,\ldots,w_r$ with $\sum_{\ell=1}^rw_\ell=1$
such that $\E_{\bz\sim\mathcal{N}_m}[g(\bz)p(\bz)]=0$ holds for every odd polynomial $p$ with degree less than $k$,
where $g(\bz)=\sum_{\ell=1}^rw_\ell \sgn(\bv_\ell^{\intercal}\bz)$.
\item $\|\bv_i+\bv_j\|_2,\|\bv_i-\bv_j\|_2\ge \Omega(\Delta), \forall1\le i<j\le r$ for some $\Delta>0$.
\end{itemize}
For an arbitrary matrix $\bU\in\R^{m\times n}$ with $\bU\bU^{\intercal}=\bI_m$,
denote by $D_{\bU}$ the instance of mixture of linear classifiers with weight vectors $\bU^{\intercal}\bv_1,\ldots,\bU^{\intercal}\bv_r$ and weights $w_1,\ldots,w_r$.
Let $\D_g=\{\D_{\bU}\mid \bU\in\R^{m\times n}, \bU\bU^{\intercal}=\bI_m\}$.
We have that $\|\bU^{\intercal}\bv_i\pm\bU^{\intercal}\bv_j\|_2=\|\bv_i\pm\bv_j\|_2\ge\Omega(\Delta),1\le i<j\le r$, and
$\E_{(\bx,y)\sim D_{\bU}}[y\mid\bx=\bz]=\sum_{\ell=1}^rw_\ell \sgn((\bU^{\intercal}\bv_\ell)^{\intercal}\mathbf{z})=\sum_{\ell=1}^rw_\ell \sgn(\bv_\ell^{\intercal}\bU\mathbf{z})=g(\bU\mathbf{z}),\bz\in\R^n.$
By~\cref{prop:low-g}, any SQ algorithm for the decision problem 
$\mathcal{B}(\D_g,\mathcal{N}_n\times\mathcal{U}(\{\pm1\}))$ must either 
use queries of tolerance $n^{-\Omega(k)}$,
or make at least $2^{n^{\Omega(1)}}$ queries.
The last step is to reduce the decision problem $\mathcal{B}(\D_g,\mathcal{N}_n\times\mathcal{U}(\{\pm1\}))$ 
to the problem of learning mixture of linear classifiers. 
\begin{claim}[see, e.g., Lemma 8.5 in~\cite{DK}]\label{clm:test-to-learn}
Suppose there exists an SQ algorithm to learn an unknown distribution in a family $\D$ to total variation distance $\eps$ using at most $N$ statistical queries of tolerance $\tau$. Suppose furthermore that for each $D'\in\D$ we have that $\dtv(D,D')>2(\tau+\eps)$. Then there exists an SQ algorithm that solves the testing problem $\mathcal{B}(\D,D)$ using at most $n+1$ queries of tolerance $\tau$.
\end{claim}

To apply~\cref{clm:test-to-learn},
we need to show that the distribution $D_{\bU}$ in the class $\mathcal{D}_g$ is sufficiently far from the null hypothesis $D_0=\mathcal{N}_n\times\mathcal{U}(\{\pm1\})$ in total variation distance.
\begin{lemma}\label{lem:TV-lower-bound}
Let $\bU\in\R^{m\times n}$ with $\bU\bU^\intercal=\bI_m$. We have that $\dtv(D_{\bU},D_0)\ge\Omega(\Delta/r)$.
\end{lemma}

\noindent We briefly sketch the proof idea and defer the proof details to~\cref{sec:omit-SQ-mixture-LTFs}.
We consider $w_\ell\sgn(\bv_\ell^\intercal\bx)$ a halfspace of heaviest weight in $D_\bU$,
and pick points $\bz$ and $\bz'$ close to the defining hyperplane that are mirrors of each other over it.
We note that under $D_0$ the expectations of $y$ conditioned on $\bx$ being $\bz$ or $\bz'$ are both 0, whereas under $D_\bU$ they likely (in particular, unless they are also on opposite sides of another halfspace) differ by at least $w_\ell$.

\section{SQ Lower Bound for Uniform Mixtures via Spherical Designs}\label{sec:SQ-mixture-LTFs-equal}
In this section, we prove our SQ lower bound for mixture of linear classifiers with uniform weights,
thereby establishing \cref{thm:main2}.
Our basic lower bound technique is essentially the same as in the previous section, but we need to construct a spherical design with {\em uniform weight}.

\begin{proposition}\label{prop:moment-matching-equal}
Let $m=\frac{c\log r}{\log(1/\Delta)}$, where $r^{-1/10}\le\Delta<1$ and $1.99\le c\le2$ is a constant.
Let $k$ be an odd integer such that $r\ge\binom{m+2k-1}{2k}^5$.
There exist vectors $\bv_1,\ldots,\bv_r$ on $\mathbb{S}^{m-1}$ such that: 
\begin{itemize}[leftmargin=*]
\item $\E_{\bz\sim\mathcal{N}_m}[g(\bz)p(\bz)]=0$ holds for every odd polynomial $p$ with degree less than $k$,
where $g(\bz)=\frac{1}{r}\sum_{\ell=1}^r\sgn(\bv_\ell^{\intercal}\bz)$.
\item $\|\bv_i+\bv_j\|_2,\|\bv_i-\bv_j\|_2\ge\Omega(\Delta)-O\big(1/\binom{m+2k-1}{2k}\big), \forall1\le i<j\le r$.
\end{itemize}
\end{proposition}

By~\cref{lem:poly}, we can write $\sum_{\ell=1}^r\E_{\bz\sim\mathcal{N}_m}\left[p(\bz)\sgn(\bv_\ell^{\intercal}\bz)\right]=\sum_{\ell=1}^r q(\bv_\ell)$ for some odd polynomial $q$ of degree less than $k$.
Therefore, it suffices to show that with high probability
there exist unit vectors $\bv_1,\ldots,\bv_r$ such that $\sum_{\ell=1}^r q(\bv_\ell)=0$ holds for every odd polynomial $q$ of degree less than $k$.
To achieve this, we will leverage some techniques
from~\cite{bondarenko2013optimal}.

\subsection{Spherical Design Construction}\label{ssec:sph-design}

\paragraph{Notation }
We start by introducing some additional notation 
we will use throughout this section.
Let $\P^d_t$ denote the set of homogeneous polynomials in $d$ variables of some odd degree $t$.
For any $p,q\in\P^d_t$, we consider the inner product $\langle p,q\rangle=\E_{\bx\sim\mathcal{U}(\mathbb{S}^{d-1})}[p(\bx)q(\bx)]$.
For any $p\in\mathcal{P}_t^d$, we define $\|p\|_2^2=\langle p,p\rangle$.
We denote by $N_{t,d}=\binom{t+d-1}{d-1}$ to be the dimension of $\P^d_t$ and $\Omega_t^d=\left\{p\in\P_t^d\mid\|p\|_2\le1\right\}$.
Let $\partial\Omega_t^d$ denote the boundary of $\Omega_t^d$, i.e., $\partial\Omega_t^d=\{p\in\P_t^d\mid \|p\|_2=1\}$.
In the remaining part of this section, we will assume that the underlying distribution (over the expectation) is $\mathcal{U}(\mathbb{S}^{d-1})$.

\medskip

We will require the following facts. 

\begin{fact}[see, e.g., Lemma 28 in~\cite{kane2015small}]\label{fact:infinity-to-l2}
For any $p\in\Omega_t^d$, we have that $$\sup_{\|\bx\|_2=1}|p(\bx)|\le\sqrt{N_{t,d}}\sqrt{\E[p(\bx)^2]}=\sqrt{N_{t,d}}\|p\|_2.$$
\end{fact}

\begin{fact}\label{fact:integral}
Let $t\ge2$ and $p,q\in\P^d_t$.
Then, we have that $$t\int_{\|\bx\|_2=1}p(\bx)q(\bx)d\bx=\frac{1}{d+2t-2}\int_{\|\bx\|_2=1}\langle\nabla p(\bx),\nabla q(\bx)\rangle d\bx+\frac{1}{d+2t-2}\int_{\|\bx\|_2=1}p(\bx)\nabla^2q(\bx)d\bx.$$
\end{fact}
\begin{proof}
Applying the Gaussian Divergence theorem for the function $p(\bx)\nabla p(\bx)$ over the unit ball, we have that
\begin{align*}
&\quad t\int_{\|\bx\|_2=1}p(\bx)q(\bx)d\bx
=\int_{\|\bx\|_2=1}\langle p(\bx)\nabla q(\bx),\bx\rangle d\bx=\int_{\|\bx\|_2\le1}\nabla\cdot (p(\bx)\nabla q(\bx))d\bx
\\&=\int_{\|\bx\|_2\le1}\langle \nabla p(\bx),\nabla q(\bx)\rangle d\bx+\int_{\|\bx\|_2\le1}p(\bx)\nabla^2q(\bx)d\bx
\\&=\int_0^1r^{d-1}dr\int_{\|\bx\|_2=1}\langle\nabla p(r\bx),\nabla q(r\bx)\rangle d\bx+\int_0^1r^{d-1}dr\int_{\|\bx\|_2=1}p(r\bx)\nabla^2 q(r\bx)d\bx
\\&=\int_0^1r^{2t+d-3}dr\int_{\|\bx\|_2=1}\langle\nabla p(\bx),\nabla q(\bx)\rangle d\bx+\int_0^1r^{2t+d-3}dr\int_{\|\bx\|_2=1}p(\bx)\nabla^2q(\bx)d\bx
\\&=\frac{1}{d+2t-2}\int_{\|\bx\|_2=1}\langle\nabla p(\bx),\nabla q(\bx)\rangle d\bx+\frac{1}{d+2t-2}\int_{\|\bx\|_2=1}p(\bx)\nabla^2q(\bx)d\bx \;.
\end{align*}
This completes the proof. 
\end{proof}

The following lemma provides upper and lower bounds for the expectation of the $L^2$-norm square of the gradient of any homogeneous polynomial $p\in\Omega_t^d$ over the unit sphere $\mathbb{S}^{d-1}$.
\begin{lemma}\label{lem:grad-bound}
Let $t$ be an odd positive integer. 
For any $p\in\P_t^d$, we have that $\E[\|\nabla_o p(\bx)\|_2^2]\ge(d-1)\|p\|_2^2$ 
and $\E[\|\nabla p(\bx)\|_2^2]\le t(d+2t-2)\|p\|_2^2$.
\end{lemma}
\begin{proof}
By~\cref{fact:integral}, we have that
\begin{align*}
t(d+2t-2)\|p\|_2^2=\E[\|\nabla p(\bx)\|_2^2]+\E[p(\bx)\nabla^2p(\bx)].
\end{align*}
We bound $\E[p(\bx)\nabla^2p(\bx)]$ as follows.
We consider the linear transformations 
$\A_t:\P^d_t\to\P^d_{t+2},\B_t:\P^d_t\to\P^d_{t-2}$ 
as follows:
$\A_t(p)=\bx^\intercal\bx p(\bx),\B_t(p)=\nabla^2 p(\bx),p\in\P_t^d$.
We first show that for any $t\ge2$, both $\A_{t-2}\B_t$ and $\B_{t+2}\A_t$ are symmetric. For any $p,q\in\P_t^d$, applying~\cref{fact:integral} yields
\begin{align*}
&\quad\langle \A_{t-2}\B_t p,q\rangle=\langle\B_{t+2}\A_t p,q\rangle=\E[\nabla^2p(\bx)q(\bx)]\\&=t(d+2t-2)\E[p(\bx)q(\bx)]-\E[\langle\nabla p(\bx),\nabla q(\bx)\rangle]\\&=\E[\nabla^2q(\bx)p(\bx)]=\langle \A_{t-2}\B_t q,p\rangle =\langle \B_{t+2}\A_t q,p\rangle.
\end{align*}
Therefore, by the eigendecomposition of symmetric linear transformations,
we have that $\lambda_1\|p\|_2^2\le\langle\A_{t-2}\B_t p, p\rangle=\E[p(\bx)\nabla^2p(\bx)]\le\lambda_t\|p\|_2^2,\forall p\in\Omega_t^d$,
where $\lambda_1\le\cdots\le\lambda_t$ denote the eigenvalues of $\A_{t-2}\B_t$.
In addition, by elementary calculation, for any $p\in\P_t^d$,
\begin{align*}
\B_{t+2}\A_t p &= \nabla^2\bx^\intercal\bx p(\bx) 
= \nabla\cdot (2p(\bx)\bx+\bx^\intercal\bx\nabla p(\bx))
= \sum_{i=1}^d \frac{\partial(2p(\bx)x_i + 
\bx^\intercal\bx(\nabla p(\bx))_i}{\partial x_i}\\
&=2dp(\bx)+4\langle\bx,\nabla p(\bx)\rangle+\bx^\intercal\bx\nabla^2p(\bx) 
= (\A_{t-2}\B_t + 2d+4t)p \;.
\end{align*}
If $\A_{t-2}\B_t$ has an eigenvector $p^*$ corresponding to some eigenvalue $\lambda^*$,
then $(\A_t\B_{t+2})(\A_tp^*) = \A_t\A_{t-2}\B_tp^*+(2d+4t)\A_tp^* = (\lambda^*+2d+4t)\A_tp^*$,
which implies that $\A_tp^*$ is an eigenvector of $\A_t\B_{t+2}$ corresponding to the eigenvalue $\lambda^*+2d+4t$.
Note that since $\B_{t+2}$ maps $\P^d_{t+2}$ to $\P^d_t$,
we have that $\ker(\B_{t+2})\ge N_{t+2,d}-N_{t,d}$,
which implies that $\A_t\B_{t+2}$ has eigenvalue 0 with multiplicity at least $N_{t+2,d}-N_{t,d}$.
Therefore, the eigenvalues of $\A_t\B_{t+2}$ are $0<\lambda_1+2d+4t\le\cdots\le\lambda_t+2d+4t$,
where the multiplicity of eigenvalue 0 is $N_{t+2,d}-N_{t,d}$ and the multiplicity of eigenvalue $\lambda_i+2d+4t$ is the same as the multiplicity of eigenvalue $\lambda_i$ of $\A_{t-2}\B_t$.
Therefore, we have that $\lambda_1=0$ and $\lambda_t=(t-1)(d+t-1)$, which implies that
\begin{align*}
\E[\|\nabla p(\bx)\|_2^2]=t(d+2t-2)\|p\|_2^2-\E[p(\bx)\nabla^2p(\bx)]\in[(t^2+d-1)\|p\|_2^2,t(d+2t-2)\|p\|_2^2].
\end{align*}
Therefore, we have that $\E[\|\nabla_o p(\bx)\|_2^2]=\E[\|\nabla p(\bx)\|_2^2-\langle\bx,\nabla p(\bx)\rangle^2]=\E[\|\nabla p(\bx)\|_2^2]-t^2\|p\|_2^2\ge(d-1)\|p\|_2^2$, 
completing the proof. 
\end{proof}


\subsubsection{Existence of Spherical Designs}

The following theorem --- a more detailed version of \cref{thm:sph-design-intro} --- 
establishes the existence 
of a spherical $t$-design of size $\poly(N_{2t,d})$ 
with the desired pairwise separation properties.


\begin{theorem}[Spherical Design Construction]\label{thm:sph-design}
Let $t$ be an odd integer and $r\ge N_{2t,d}^5$.
Let $\by_1,\ldots,\by_r$ be uniform random vectors over $\mathbb{S}^{d-1}$.
Then, with probability at least 99/100, there exist unit vectors 
$\bz_1,\ldots,\bz_r\in\mathbb{S}^{d-1}$ such that $\|\bz_i-\by_i\|_2\le O(1/N_{2t,d}),i\in[r]$, 
and $(\bz_1,\ldots,\bz_r)$ form a spherical $t$-design.
\end{theorem}

\noindent To prove \Cref{thm:sph-design}, we will start from the following result.

\begin{theorem}[\cite{bondarenko2013optimal}]\label{thm:exist-sph-design}
If there exists a continuous mapping $F:\P^d_t\to(\mathbb{S}^{d-1})^r$
such that for all $p\in\partial\Omega_t^d$, $\sum_{i=1}^r p(\bx_i(p))>0$, 
where $F(p)=(\bx_1(p),\ldots,\bx_r(p))$, then there exists a polynomial 
$p^*\in\Omega_t^d$ such that 
$\E[p(\bx)]=\frac{1}{r}\sum_{i=1}^rp(\bx_i(p^*))$ holds 
for every polynomial $p\in\P^d_t$.
\end{theorem}

To apply~\cref{thm:exist-sph-design}, we need to find a continuous function $F$ mapping $\P_t^d$ to $(\mathbb{S}^{d-1})^r$ such that for any $p\in\partial\Omega_t^d$, $\sum_{i=1}^r p(\bz_i) > 0$,
where $F(p) = (\bz_1,\ldots,\bz_r)$.
We will construct the mapping $F$ as follows:
we sample $\by_1,\ldots,\by_r$ uniformly over the unit sphere $\mathbb{S}^{d-1}$, and then try to make the value of $p(\by_i)$ larger by moving each point $\by_i$ in the direction of the gradient. In particular, we let $\bz_i=\frac{\by_i+\delta\nabla_o p(\by_i)}{\|\by_i+\delta\nabla_o p(\by_i)\|_2}$ for some $\delta>0$ sufficiently small,
where $\nabla_o p(\by)$ is the component of $\nabla p(\by)$ orthogonal to the direction $\by$.
We will prove that with high probability,
for any $p\in\partial\Omega_t^d$,
$\sum_{i=1}^rp(\bz_i)>0$.
Intuitively, this works because of two facts:
\begin{enumerate}[leftmargin=*]
\item With high probability over the choice of $\by_i$, for all $p\in\Omega_t^d$, the average value of $p(\by)$ is already close to zero.
\item
Moving in the direction of $\nabla_o p(\by_i)$ increases $p(\by_i)$ by a notable amount.
\end{enumerate}

\begin{lemma}\label{lem:poly-concentration}
Let $\Omega$ be a subspace of polynomials in $d$ variables with mean zero. Let $N$ be the dimension of $\Omega$. Let $\bx_1,\ldots,\bx_r$ be i.i.d. random vectors over $\mathbb{S}^{d-1}$. Then, with probability at least $1-\frac{N}{r\eta^2}$, we have that
for any $p\in\Omega$,
$\left|\frac{1}{r}\sum_{i=1}^rp(\bx_i)\right|\le\eta\|p\|_2$.
\end{lemma}
To prove~\cref{lem:poly-concentration}, 
we take an orthonormal basis $p_1,\ldots,p_N\in\Omega$
and define $\bp(\bx)\eqdef [p_1(\bx),\ldots,p_N(\bx)]^\intercal$. 
Noting that $\bp(\bx)$ is a random vector with mean zero 
and covariance identity, applying Markov's inequality 
will yield the result. We defer the proof details 
to~\cref{sec:omit-SQ-mixture-LTFs-equal}.

\begin{lemma}\label{lem:increment}
Let $\by\in\mathbb{S}^{d-1}$ and $0<\delta\le1/N_{2t,d}^2$.
For any $p\in\partial\Omega_t^d$, let $\bz=\frac{\by+\delta\nabla_o p(\by)}{\|\by+\delta\nabla_o p(\by)\|_2}$. We have that $p(\bz)-p(\by)\ge C\delta\|\nabla_op(\by)\|_2^2$ for some universal constant $0<C<1$.
\end{lemma}
The proof of~\cref{lem:increment} follows by Taylor expansion, 
where the contributions to $p(\bz)$ coming from higher-order terms 
will not be large as long as $\bz$ is sufficiently close to $\by$. 
We defer the proof details to~\cref{sec:omit-SQ-mixture-LTFs-equal}.
By applying the above two lemmas, we establish the following:

\begin{theorem}\label{thm:sph-construction}
Let $\by_1,\ldots,\by_r$ be i.i.d.\ random vectors over $\mathbb{S}^{d-1}$. Let $\delta=1/N_{2t,d}^2$ and $r\ge N_{2t,d}^5$. We consider the mapping $F:\P^d_t\to(\mathbb{S}^{d-1})^r$ as follows: for any $p\in\P^d_t$, let $$\bz_i=\frac{\by_i+\delta\nabla_o p(\by_i)}{\|\by_i+\delta\nabla_o p(\by_i)\|_2} \;,$$ 
where $\nabla_o p(\by)$ is the component of $\nabla p(\by)$ 
orthogonal to the direction $\by$. 
Let $F(p):=(\bz_1,\ldots,\bz_r)$. Then, with probability at least 99/100,
we have that for any $p\in\partial\Omega_t^d$,
$\sum_{i=1}^rp(\bz_i)>0$.
\end{theorem}

To prove \cref{thm:sph-construction},  we 
consider $\wt{p}(\by)=p(\by)+C\delta(\|\nabla_op(\by)\|_2^2-\E[\|\nabla_op(\by)\|_2^2])$, where $C,\delta$ comes 
from~\cref{lem:increment}. Noting that $\E[\wt{p}(\by)]=0$ and $\wt{p}(\by)$ 
is a polynomial containing only monomials of degree $2t,2t-2,t,0$, we are 
able to apply~\cref{lem:poly-concentration} to obtain the desired result. 
See \cref{sec:omit-SQ-mixture-LTFs-equal} for the proof.

\subsubsection{Proof of~\cref{thm:sph-design}}


Let $\delta=1/N_{2t,d}^2$.
We consider the mapping $F:\P^d_t\to(\mathbb{S}^{d-1})^r$ as follows: for any $p\in\P^d_t$, let $\bz_i=\frac{\by_i+\delta\nabla_o p(\by_i)}{\|\by_i+\delta\nabla_o p(\by_i)\|_2}$, where $\nabla_o p(\by)$ is the component of $\nabla p(\by)$ orthogonal to the direction $\by$.
By~\cref{thm:sph-construction}, with probability at least 99/100, 
we have that
for any $p\in\partial\Omega_t^d$,
$\sum_{i=1}^rp(\bz_i)>0$.
Applying~\cref{thm:exist-sph-design} yields that there exists some 
$p^*\in\Omega_t^d$ such that $F(p^*)=(\bz_1^*,\ldots,\bz_r^*)$ form a 
spherical $t$-design.
By elementary calculation, we have that
\begin{align*}
\|\bz_i^*-\by_i\|_2&=\left\|\frac{\by_i+\delta\nabla_o p^*(\by_i)}{\|\by_i+\delta\nabla_o p^*(\by_i)\|_2}-\by_i\right\|_2=\frac{\|\by_i+\delta\nabla_o p^*(\by_i)-\|\by_i+\delta\nabla_o p^*(\by_i)\|_2\by_i\|_2}{\|\by_i+\delta\nabla_o p^*(\by_i)\|_2}
\\&\le\frac{|1-\|\by_i+\delta\nabla_o p^*(\by_i)\|_2|+\delta\|\nabla p^*(\by_i)\|_2}{1-\delta\|\nabla p^*(\by_i)\|_2}\le\frac{2\delta\|\nabla p^*(\by_i)\|_2}{1-\delta\|\nabla p^*(\by_i)\|_2}\le O(1/N_{2t,d}) \;,
\end{align*}
where the last inequality follows from $\forall\by\in\mathbb{S}^{d-1},\|\nabla p^*(\by)\|_2\le\sqrt{t(d+2t-2)N_{2(t-1),d}\|p^*\|_2^2}\le N_{2t,d}$ by~\cref{lem:grad-bound}.

\subsection{Proof of~\cref{thm:main2}}
We first prove~\cref{prop:moment-matching-equal} based on our construction of spherical $t$-design.
To achieve this, it suffices to show that with high probability there exist $\bv_1,\ldots,\bv_r\in\mathbb{S}^{m-1}$ such that
\begin{itemize}[leftmargin=*]
\item $\bv_1,\ldots,\bv_r$ is a spherical $k$-design.
\item $\|\bv_i+\bv_j\|_2,\|\bv_i-\bv_j\|_2\ge\Omega(\Delta), \forall1\le i<j\le r$ for some $\Delta>0$.
\end{itemize}

\begin{proof}[Proof of~\cref{prop:moment-matching-equal}]
Let $\Theta = \min_{1\le i<j\le r}\arccos\langle \by_i,\by_j\rangle$. 
Applying~\cref{fact:arg} by taking $\gamma=\left(\frac{1}{50\kappa_{m-1}}\right)^{\frac{1}{m-1}}$ yields $
\pr[\Theta\ge\Delta]\ge\pr\left[\Theta\ge\gamma r^{-\frac{2}{m-1}}\right]\ge99/100$. 
This will give  
$$\pr\Big[\min_{1\le i<j\le r}\{\|\by_i-\by_j\|_2,\|\by_i+\by_j\|_2\}\ge\Omega(\Delta)\Big]\ge99/100 \;,$$
since our choice of $m$ will imply that
$r=1/\sqrt{50\kappa_{m-1}\Delta^{m-1}}$.
By~\cref{thm:sph-design}, with probability at least 99/100,
there exist unit vectors $\bz^*_1,\ldots,\bz^*_r\in\mathbb{S}^{m-1}$ such that
$(\bz^*_1,\ldots,\bz^*_r)$ form a spherical $k$-design and
$\|\bz^*_i-\by_i\|_2\le O(1/N_{2k,m}),i\in[r]$.
Therefore, for any odd homogeneous polynomial $p$ in $m$ variables of odd degree $t<k$,
we have that 
$$\frac{1}{r}\sum_{i=1}^rp(\bz^*_i)=\frac{1}{r}\sum_{i=1}^r(\|\bz_i^*\|_2^2)^\frac{k-t}{2}p(\bz^*_i)=\E\left[(\|\bz\|_2^2)^\frac{k-t}{2}p(\bz)\right]=0 \;,$$
which implies that $\sum_{i=1}^rq(\bz_i^*)=0$ holds for any polynomial $q$ in $m$ variables of degree less than $k$.
In addition, for every $1\le i<j\le r$, we have that 
\begin{eqnarray*}
\|\bz_i^*\pm\bz_j^*\|_2 
&=&\|(\by_i\pm\by_j)+(\bz^*_i-\by_i)\pm(\bz^*_j-\by_j)\|_2\ge\|\by_i\pm\by_j\|-\|\bz_i^*-\by_i\|_2-\|\bz_j^*-\by_j\|_2 \\ 
&\ge&\|\by_i\pm\by_j\|_2-O(1/N_{2k,m}) \;.
\end{eqnarray*}
This completes the proof.
\end{proof}

\begin{proof}[Proof of~\cref{thm:main2}]
Let $m=\frac{c\log r}{\log(1/\Delta)}$, where $r^{-1/10}\le\Delta<1$ and $1.99\le c\le2$ is a constant.
Let $c'=\frac{(1/e)(1/\Delta)^{1/(5c)}-1}{2}$ and $k=c'm=\frac{\left((1/e)(1/\Delta)^{1/(5c)}-1\right)c\log r}{2\log(1/\Delta)}$.
In this way, we will have that
\begin{align*}
N_{2k,m}^5&\le\binom{m+2k}{2k}^5=\binom{(1+2c')m}{m}^5\le2^{5(1+2c')mH\left(\frac{1}{1+2c'}\right)}
=2^{5(1+2c')m\left(\frac{\log(1+2c')}{1+2c'}+\frac{2c'\log(1+1/2c')}{1+2c'}\right)}\\&=2^{5m\left(\log(1+2c')+2c'\log\left(1+1/2c'\right)\right)}\le2^{\frac{5c\log r(\log(1+2c')+2c'\log(1+1/2c'))}{\log(1/\Delta)}}\le r,
\end{align*}
where $H(p)=-p\log p-(1-p)\log(1-p)$, $p\in[0,1]$, is the standard binary entropy function.
Therefore, by~\cref{prop:moment-matching-equal}, there exist vectors $\bv_1,\ldots,\bv_r\in\mathbb{S}^{m-1}$ such that
\begin{itemize}[leftmargin=*]
\item $\E_{\bz\sim\mathcal{N}_m}[g(\bz)p(\bz)]=0$ holds for every odd polynomial $p$ of degree less than $k$,
where $g(\bz)=\frac{1}{r}\sum_{\ell=1}^r \sgn(\bv_\ell^{\intercal}\bz)$.
\item $\|\bv_i+\bv_j\|_2,\|\bv_i-\bv_j\|_2\ge\Omega(\Delta)-O(1/N_{2k,m}), \forall1\le i<j\le r$.
\end{itemize}
For an arbitrary matrix $\bU\in\R^{m\times n}$ with $\bU\bU^{\intercal}=\bI_m$,
denote by $D_{\bU}$ the instance of mixture of linear classifiers with weight vectors $\bU^{\intercal}\bv_1,\ldots,\bU^{\intercal}\bv_r$.
Let $\D_g=\{\D_{\bU}\mid \bU\in\R^{m\times n}, \bU\bU^{\intercal}=\bI_m\}$.
By definition, we have that $\E_{(\bx,y)\sim D_{\bU}}[y\mid\bx=\bz]=\frac{1}{r}\sum_{\ell=1}^r \sgn((\bU^{\intercal}\bv_\ell)^{\intercal}\mathbf{z})=\frac{1}{r}\sum_{\ell=1}^r \sgn(\bv_\ell^{\intercal}\bU\mathbf{z})=g(\bU\mathbf{z}),\bz\in\R^n$,
and $\|\bU^{\intercal}\bv_i\pm\bU^{\intercal}\bv_j\|_2=\|\bv_i\pm\bv_j\|_2\ge\Omega(\Delta),1\le i<j\le r$, since
\begin{align*}
N_{2k,m}&=\binom{m+2k-1}{m-1}=\binom{(1+2c')m-1}{m-1} 
\ge\left(\frac{(1+2c')m-1}{m-1}\right)^{m-1}\ge(1+2c')^{m-1}\\ 
&\ge(1+2c')^{\frac{1.99\log r}{\log(1/\Delta)}-1} 
=\left((1/e)(1/\Delta)^{1/5c}\right)^{\frac{1.89\log r}{\log(1/\Delta)}}\ge\Omega((1/\Delta)^{1.89}) \;.
\end{align*}
Therefore by~\cref{prop:low-g}, any SQ algorithm that solves the decision problem $\mathcal{B}(\D_g,\mathcal{N}_n\times\mathcal{U}(\{\pm1\}))$ must either use queries of tolerance $n^{-\Omega(k)}$,
or make at least $2^{n^{\Omega(1)}}$ queries.
By~\cref{lem:TV-lower-bound}, for any $\bU\in\R^{m\times n}$ with $\bU\bU^\intercal=\bI_m$, we have that
$\dtv(D_{\bU},D_0)\ge\Omega(\Delta/r)\ge2(n^{-\Omega(k)}+\epsilon)$, where $D_0=\mathcal{N}_n\times\mathcal{U}(\{\pm1\})$.
Therefore, by~\cref{clm:test-to-learn},
any SQ algorithm that learns a distribution in $\D_g$ within error $\epsilon$ in total variation distance must either use queries of tolerance 
$n^{-\Omega (\log r/ (\Delta^{1/(5c)}\log(1/\Delta)))}$,
or make at least $2^{n^{\Omega(1)}}$ queries.
This completes the proof.
\end{proof}

\section{Conclusion}\label{sec:con}
This work establishes a near-optimal Statistical Query (SQ) lower bound for learning uniform mixtures of linear classifiers under the Gaussian distribution. Our lower bound nearly matches prior algorithmic work on the problem.
Our result applies for the simplest (and well-studied) distributional setting where the covariates are drawn from the standard Gaussian distribution.
This directly implies similar information-computation tradeoffs for the setting that the covariates are drawn from a more general distribution family (e.g., an unknown subgaussian or a log-concave distribution) that includes the standard normal.

From a technical perspective, we believe that our new efficient construction of spherical designs is a mathematical contribution of independent interest that could be used to establish SQ lower bounds for other related latent variable models (e.g., various mixtures of experts).
A natural direction is to establish information-computation tradeoffs for a fixed non-Gaussian distribution on covariates (e.g., the uniform distribution over the Boolean hypercube), for which a different hardness construction is needed.


\bibliographystyle{alpha}

\bibliography{main}

\newpage

\appendix

\section*{Appendix} \label{sec:app}

\section{Background on Hermite Polynomials}\label{sec:background}

Recall the definition of the probabilist's Hermite polynomials:
\begin{align*}
\mathrm{\textit{He}}_n(x)=(-1)^ne^{x^2/2}\cdot\frac{d^2}{dx^2}e^{-x^2/2}.
\end{align*}
Under this definition, the first four Hermite polynomials are
\begin{align*}
\mathrm{\textit{He}}_0(x)=1, \mathrm{\textit{He}}_1(x)=x, \mathrm{\textit{He}}_2(x)=x^2-1,\mathrm{\textit{He}}_3(x)=x^3-3x.
\end{align*}
In our work, we will consider the \emph{normalized} Hermite polynomial of degree $n$ to be $h_n(x)=\frac{\mathrm{\textit{He}}_n(x)}{\sqrt{n!}}$. These normalized Hermite polynomials form a complete orthogonal basis for inner product space $\L^2(\R,\mathcal{N})$. To obtain an orthogonal basis for $\L^2(\R^d,\mathcal{N}_d)$, we will use a multi-index $J=(j_1,\ldots,j_d)\in\mathbb{N}^d$ to define the $d$-variate normalized Hermite polynomial as $H_J(\bx) = \prod_{i=1}^dH_{j_i}(x_i)$. Let the total degree of $H_J$ be $|J|=\sum_{i=1}^d j_i$. Given a function $f\in\L^2(\R^d,\mathcal{N}_d)$, we can express it uniquely as $f(\bx)=\sum_{J\in\N^d}\wh{f}(J)H_J(\bx)$, where $\wh{f}(J)=\E_{\bx\in\mathcal{N}_d}[f(\bx)H_J(\bx)]$.
We denote by $f^{[k]}(\bx)$ the degree $k$ part of the Hermite expansion of $f$, i.e., $f^{[k]}(\bx)=\sum_{|J|=k}\wh{f}(J)H_J(\bx)$.
\begin{definition}\label{def:harmonic}
We say that a polynomial $q$ in $d$ variables is harmonic of degree $k$
if it is a linear combination of degree $k$ Hermite polynomials. 
That is, $q$ is harmonic if it can be written as
\begin{align*}
q(\bx)=q^{[k]}(\bx)=\sum_{J:|J|=k}c_JH_J(\bx).
\end{align*}
\end{definition}

Notice that, since for a single-dimensional Hermite polynomial 
it holds $h'_m(x)=\sqrt{m}h_{m-1}(x)$,
we have that $\nabla H^{(i)}_M(\bx)=\sqrt{m_i}H_{M-E_i}(\bx)$,
where $M=(m_1,\ldots,m_d)$.
From this fact and the orthogonality of Hermite polynomials, we obtain
\begin{align*}
\E_{\bx\sim\mathcal{N}_d}[\langle\nabla H_M(\bx),\nabla H_L(\bx)\rangle]=|M| \; \I[M=L] \;.
\end{align*}

We will also require the following standard facts:

\begin{fact}\label{fact:diff-harm}
Let $p$ be a polynomial of degree $k$ in $d$ variables. Then $p$ is harmonic of degree $k$ if and only if for all $\bx\in\R^d$ it holds that
$kp(\bx) = \langle\bx,\nabla p(\bx)\rangle-\nabla^2p(\bx)$.
\end{fact}

\begin{fact}[see, e.g.,~\cite{DiakonikolasKPZ21}]\label{fact:grad}
Let $p,q$ be harmonic polynomials of degree $k$. Then,
\begin{align*}
\E_{\bx\sim\mathcal{N}_d}\left[\langle\nabla^\ell p(\bx),\nabla^\ell q(\bx)\rangle\right]=k(k-1)\ldots(k-\ell+1)\E_{\bx\sim\mathcal{N}_d}[p(\bx)q(\bx)].
\end{align*}
In particular,
\begin{align*}
\langle\nabla^kp(\bx),\nabla^kq(\bx)\rangle=k!\E_{\bx\sim\mathcal{N}_d}[p(\bx)q(\bx)].
\end{align*}
\end{fact}

\section{Omitted Proofs from~\cref{sec:SQ-mixture-LTFs}}\label{sec:omit-SQ-mixture-LTFs}

\subsection{Proof of~\cref{lem:poly}}\label{ssec:poly}

We start with the following claim:

\begin{claim}\label{clm:inner-prod-ext}
Let $p:\R^{n_1}\to\R$ and $q:\R^{n_2}\to\R$,
where $p$ is a polynomial of degree at most $k$ and $q\in\L^2(\R^{n_2},\mathcal{N}_{n_2})$.
Let $\bU\in\R^{n_1\times n},\bV\in\R^{n_2\times n}$ such that $\bU\bU^\intercal=\mathbf{I}_{n_1},\bV\bV^\intercal=\mathbf{I}_{n_2}$.
Then, we have that $\E_{\bx\sim\mathcal{N}_n}[p(\bU\bx)q(\bV\bx)]=\sum_{m=0}^k\frac{1}{m!}\langle (\bU^{\intercal})^{\otimes m}\mathbf{R}^m_1,(\bV^{\intercal})^{\otimes m}\mathbf{R}^m_2\rangle$,
where $\mathbf{R}^m_1=\nabla^mp^{[m]}(\bx),\mathbf{R}^m_2=\nabla^mq^{[m]}(\bx)$.
\end{claim}

We require the following lemma:
\begin{lemma}\label{lem:k-harm}
Let $p$ be a harmonic polynomial of degree $k$.
Let $\bV\in\R^{m\times n}$ with $\bV\bV^\intercal=\bI_m$. Then the polynomial $p(\bV\bx)$ is harmonic of degree $k$.
\end{lemma}
\begin{proof}
Let $f(\bx)=p(\bV\bx)$.
By~\cref{fact:diff-harm}, it suffices to show that for all $\bx\in\R^n$ it holds that $kf(\bx)=\langle\bx,\nabla f(\bx)\rangle-\nabla^2f(\bx)$. Since $\bV\bV^\intercal=\bI_m$, applying~\cref{fact:diff-harm} yields \[ \langle\bx,\nabla f(\bx)\rangle-\nabla^2f(\bx)=\langle\bV\bx,\nabla p(\bV\bx)\rangle-\nabla^2p(\bV\bx)=kp(\bV\bx)=kf(\bx) \;. \]
This completes the proof.
\end{proof}

\begin{proof}[Proof of~\cref{clm:inner-prod-ext}]
For $m\in\N$,
let $f^{(m)}(\bx)=p^{[m]}(\bU\bx)$ and $g^{(m)}(\bx)=q^{[m]}(\bV\bx)$. We can write $p(\bU\bx)\sim\sum_{m=0}^{k}f^{(m)}(\bx)$ and $q(\bV\bx)\sim\sum_{m=0}^\infty g^{(m)}(\bx)$.
Then applying~\cref{fact:grad} and~\cref{lem:k-harm} yields
\begin{align*}
\E_{\bx\sim\mathcal{N}_n}[p(\bU\bx)q(\bV\bx)]&=\sum_{m_1=0}^k\sum_{m_2=0}^\infty\E_{\bx\sim\mathcal{N}_n}[f^{(m_1)}(\bx)g^{(m_2)}(\bx)]=\sum_{m=0}^k\E_{\bx\sim\mathcal{N}_n}[f^{(m)}(\bx)g^{(m)}(\bx)]
\\&=\sum_{m=0}^k\frac{1}{m!}\left\langle\nabla^mf^{(m)}(\bx),\nabla^mg^{(m)}(\bx)\right\rangle=\sum_{m=0}^k\frac{1}{m!}\left\langle\nabla^mp^{[m]}(\bU\bx),\nabla^mq^{[m]}(\bV\bx)\right\rangle \;.
\end{align*}
Denote by $\U\subseteq\R^n$ the image of the linear map $\bU^{\intercal}$.
Applying the chain rule, for any function $h(\bU\bx):\R^n\to\R$,
it holds $\nabla h(\bU\bx)=\partial_ih(\bU\bx)U_{ij}\in\U$,
where we applied Einstein's summation notation for repeated indices.
Applying the above rule $m$ times, we have that
\begin{align*}
\nabla^mh(\bU\bx)=\partial_{i_m}\ldots\partial_{i_1}h(\bU\bx)U_{i_1,j_1}\ldots U_{i_m,j_m}\in\U^{\otimes m}.
\end{align*}
Moreover, denote $\S_m=\nabla^mp^{[m]}(\bU\bx)=(\bU^{\intercal})^{\otimes m}\mathbf{R}_1^m\in\U^{\otimes m}$,
and $\T_m=\nabla^mq^{[m]}(\bV\bx)=(\bV^{\intercal})^{\otimes m}\mathbf{R}^m_2\in\V^{\otimes m}$.
We have that
\begin{align*}
\E_{\bx\sim\mathcal{N}_n}[f(\bx)g(\bx)]&=\sum_{m=0}^k\frac{1}{m!}\left\langle\nabla^mp^{[m]}(\bU\bx),\nabla^mq^{[m]}(\bV\bx)\right\rangle=\sum_{m=0}^k\frac{1}{m!}\langle \S_m,\T_m\rangle
\\&=\sum_{m=0}^k\frac{1}{m!}\langle (\bU^{\intercal})^{\otimes m}\mathbf{R}^m_1,(\bV^{\intercal})^{\otimes m}\mathbf{R}^m_2\rangle.
\end{align*}
This proves the claim. 
\end{proof}

\begin{proof}[Proof of~\cref{lem:poly}]
Applying~\cref{clm:inner-prod-ext} by taking $\bU=\bI_m$ and $\bV=\bv^\intercal$, we have that
\begin{align*}
\E_{\bz\sim\mathcal{N}_m}\left[p(\bz)f(\bv^{\intercal}\bz)\right]=\sum_{d=0}^{k-1}\frac{1}{d!}\langle\mathbf{R}^d_1,\bv^{\otimes d}\mathbf{R}^d_2\rangle,
\end{align*}
which is a polynomial in $\bv$ of degree less than $k$,
since $\mathbf{R}^d_1=\nabla^dp^{[d]}(\bx)$ and $\mathbf{R}^d_2=\nabla^df^{[d]}(\bx)$ 
are constants only depending on $p$ and $f$. This completes the proof of \cref{lem:poly}. 
\end{proof}

\subsection{Proof of~\cref{lem:equiv}}\label{ssec:equiv}
We start by proving that {\em ``there exist non-negative weights $w_1,\ldots,w_r$ 
with $\sum_{\ell=1}^rw_\ell=1$ such that $\sum_{\ell=1}^r w_\ell q(\bv_\ell)=0$ 
for all odd polynomials $q$ of degree less than $k$''} 
implies {\em ``there does not exist any odd polynomial $q$ of degree less than $k$ 
such that $q(\bv_\ell)>0,1\le\ell\le r$.'' }
Suppose for contradiction that there exists an odd polynomial $q^*$ 
of degree less than $k$ such that $q^*(\bv_\ell)>0,1\le\ell\le r$.
For arbitrary non-negative weights $w_1,\ldots,w_r$ with $\sum_{\ell=1}^rw_\ell=1$,
we have that $\sum_{\ell=1}^rw_\ell q^*(\bv_\ell)\ge\min\{q^*(\bv_1),\ldots,q^*(\bv_r)\}>0$,
which contradicts to the first statement.

We then prove the opposite direction. 
We will use the following version of Farkas' lemma.
\begin{fact}[Farkas' lemma]\label{fact:Farkas}
Let $A\in \R ^{m\times n}$ and $b\in\R^m$. Then exactly one of the following two assertions is true:
\begin{itemize}[leftmargin=*]
\item There exists an $\bx\in\R^n$ such that $A\bx=b$ and $\bx\ge0$.
\item There exists a $\by\in\R^m$ such that $\by^{\intercal}A\ge0$ and $\by^{\intercal}b<0$.
\end{itemize}
\end{fact}
Suppose for contradiction that there does not exist $w_1,\ldots,w_r$ with $\sum_{\ell=1}^rw_\ell=1$
such that $\sum_{\ell=1}^r w_\ell q(\bv_\ell)=0$ holds for every odd polynomial $q$ of degree less than $k$.
Let $s_{k,m}$ denote the total number of $m$-variate odd monomials of degree less than $k$,
and $\{q^{k,m}_j\}_{1\le j\le s_{k,m}}$ denote such monomials.
We consider the following LP with variables $\bw=(w_1,\ldots,w_r)^{\intercal}$: $\sum_{\ell=1}^rw_\ell q^{k,m}_j(\bv_\ell)=0,1\le j\le s_{k,m},\sum_{\ell=1}^r w_\ell=1,w_\ell\ge0,1\le\ell\le r$.
By our assumption, the LP is infeasible.
In order to applying the Farkas Lemma (\cref{fact:Farkas}), we write the linear system as $A\bw=b$, where
\begin{align*}
\mathbf{A} = 
\begin{bmatrix}
1 & 1 & \cdots & 1 \\
q^{k,m}_1(\bv_1) & q^{k,m}_1(\bv_2) & \cdots & q^{k,m}_1(\bv_r) \\
\vdots & \vdots &\ddots & \vdots \\
q^{k,m}_{s_{k,m}}(\bv_1) & q^{k,m}_{s_{k,m}}(\bv_2) & \cdots & q^{k,m}_{s_{k,m}}(\bv_r)
\end{bmatrix}
, \bw =
\begin{bmatrix}
w_1\\
w_2\\
\vdots\\
w_r
\end{bmatrix}
, \bb =
\begin{bmatrix}
1\\
0\\
\vdots\\
0
\end{bmatrix}
.
\end{align*}
By~\cref{fact:Farkas}, the original linear system is infeasible if and only if there exists a vector $\bu=[u_0,u_1,\ldots,u_{s_{k,d}}]^{\intercal}$,
$\bu^{\intercal}\mathbf{A}\ge0$ and $\bu^{\intercal}\bb<0$,
which is equivalent to $u_0+\sum_{j=1}^{s_{k,m}}u_jq^{k,m}_j(\bv_\ell)\ge0,\forall1\le\ell\le r$ and $u_0<0$.
Let $q^*(\bv)=\sum_{j=1}^{s_{k,m}}u_jq^{k,m}_j(\bv),\bv\in\R^m$,
which is an odd polynomial of degree less than $k$.
By our definition of $q^*$,
we have that $q^*(\bv_\ell)=\sum_{j=1}^{s_{k,m}}u_jq^{k,m}_j(\bv_\ell)\ge-u_0>0,\forall1\le\ell\le r$,
which contradicts to our assumption that {\em there does not exist 
any odd polynomial $q$ of degree less than $k$ such that 
$q(\bv_\ell)>0,\forall1\le\ell\le r$}. This completes the proof.

\subsection{Proof of~\cref{lem:TV-lower-bound}}\label{ssec:TV-lower-bound}
We denote by $G(\bx)$ to be the standard Gaussian density. By definition, we have that
\begin{align*}
&\quad\dtv(D_{\bU},D_0)=(1/2)\int_{\bx\in\R^n}\sum_{y\in\{\pm1\}}|D_{\bU}(\bx,y)-D_0(\bx,y)|d\bx
\\&=(1/2)\int_{\bx\in\R^n}G(\bx)\sum_{y\in\{\pm1\}}\left|\sum_{\ell=1}^rw_\ell\mathbb{I}[\sgn(\bv_\ell^{\intercal}\bU\bx)=y]-(1/2)\right|d\bx
\\&=(1/2)\E_{\bx\sim\mathcal{N}_n}\left[\sum_{y\in\{\pm1\}}\left|\sum_{\ell=1}^rw_\ell\mathbb{I}[\sgn(\bv_\ell^{\intercal}\bU\bx)=y]-(1/2)\right|\right]\\&=(1/2)\sum_{y\in\{\pm1\}}\E_{\bx\sim\mathcal{N}_n}\left[\left|\sum_{\ell=1}^rw_\ell\mathbb{I}[\sgn(\bv_\ell^{\intercal}\bU\bx)=y]-(1/2)\right|\right].
\end{align*}
Therefore, it suffices to show that
$$\E_{\bx\sim\mathcal{N}_n}\left[\left|\sum_{\ell=1}^rw_\ell\mathbb{I}[\sgn(\bv_\ell^{\intercal}\bU\bx)=y]-(1/2)\right|\right]\ge\Omega(\Delta/r),\quad \forall y\in\{\pm1\}.$$
We assume that $w_{\ell_0}\ge1/r$ for some $\ell_0\in[r]$. Let $\bv^*$ be an arbitrary vector satisfying $\bv_{\ell_0}^{\intercal}\bv^*=0$.
We denote by
\begin{align*}
\mathcal{X}_1=\{\bx\in\R^m\mid\sgn(\bv_{\ell_0}^{\intercal}\bx)>0,\sgn(\bv_\ell^{\intercal}\bx)=\sgn(\bv_\ell^{\intercal}\bv^*),\ell\in[r]\setminus\{\ell_0\}\},\\
\mathcal{X}_2=\{\bx\in\R^m\mid\sgn(\bv_{\ell_0}^{\intercal}\bx)<0,\sgn(\bv_\ell^{\intercal}\bx)=\sgn(\bv_\ell^{\intercal}\bv^*),\ell\in[r]\setminus\{\ell_0\}\}.
\end{align*}
Roughly speaking, $\mathcal{X}_1$ and $\mathcal{X}_2$ denote the subsets of vectors
which are very close to the boundary of the halfspace with direction $\bv_{\ell_0}$
and maintain the same label with the boundary for the other halfspaces.
By definition, for any $\bx_1\in\mathcal{X}_1,\bx_2\in\mathcal{X}_2$,
we have that $$\left|\sum_{\ell=1}^rw_\ell\mathbb{I}[\sgn(\bv_\ell^{\intercal}\bx_1)=y]-\sum_{\ell=1}^rw_\ell\mathbb{I}[\sgn(\bv_\ell^{\intercal}\bx_2)=y]\right|=w_{\ell_0}\ge1/r.$$
Therefore, we have either $$\left|\sum_{\ell=1}^rw_\ell\mathbb{I}[\sgn(\bv_\ell^{\intercal}\bx_1)=y]-(1/2)\right|\ge1/2r,\quad\forall\bx_1\in\mathcal{X}_1,$$
or $$\left|\sum_{\ell=1}^rw_\ell\mathbb{I}[\sgn(\bv_\ell^{\intercal}\bx_2)=y]-(1/2)\right|\ge1/2r,\quad\forall\bx_2\in\mathcal{X}_2.$$
Since $\bU\bx$ is a standard Gaussian for any $\bU\bU^{\intercal}=\bI_m$ and $\|\bv_i+\bv_j\|_2,\|\bv_i-\bv_j\|_2\ge\Omega(\Delta),1\le i<j\le r$, we have that for $y\in\{\pm1\}$,
\begin{align*}
&\quad\E_{\bx\sim\mathcal{N}_n}\left[\left|\sum_{\ell=1}^rw_\ell\mathbb{I}[\sgn(\bv_\ell^{\intercal}\bU\bx)=y]-(1/2)\right|\right]\\&\ge\pr_{\bx\sim\mathcal{N}_n}[\bU\bx\in\mathcal{X}_1]\cdot\E_{\bx\sim\mathcal{N}_n}\left[\left|\sum_{\ell=1}^rw_\ell\mathbb{I}[\sgn(\bv_\ell^{\intercal}\bU\bx)=y]-(1/2)\right|\mid\bU\bx\in\mathcal{X}_1\right]\\&\quad+\pr_{\bx\sim\mathcal{N}_n}[\bU\bx\in\mathcal{X}_2]\cdot\E_{\bx\sim\mathcal{N}_n}\left[\left|\sum_{\ell=1}^rw_\ell\mathbb{I}[\sgn(\bv_\ell^{\intercal}\bU\bx)=y]-(1/2)\right|\mid\bU\bx\in\mathcal{X}_2\right]\\&\ge\Omega(\Delta/r).
\end{align*}

\section{Omitted Proofs from~\cref{sec:SQ-mixture-LTFs-equal}}\label{sec:omit-SQ-mixture-LTFs-equal}

We start by introducing the following technical result which provides a universal upper bound 
for the $L_2^2$-norm of the high-order gradient of any homogeneous polynomial $p\in\Omega_t^d$. 

\begin{lemma}\label{lem:higher-order-max}
For any $p\in\Omega_t^d$ and any $1\le j\le t$,
we have that $$\sup_{\|\bx\|_2=1}\left\|\frac{\partial^jp(\by)}{\partial\by^j}\right\|_2^2\le t^j(d+2t-2)^jN_{2(t-j),d}\|p\|_2^2.$$
\end{lemma}

\begin{proof}
Note that $\|\nabla p(\bx)\|_2^2\in\Omega^d_{2(t-1)}$, by~\cref{fact:infinity-to-l2}, we have that
\begin{align*}
\sup_{\|\bx\|_2=1}\|\nabla p(\bx)\|_2^{2}\le\sqrt{N_{2(t-1),d}}\sqrt{\E[\|\nabla p(\bx)\|_2^{4}]}\le\sqrt{N_{2(t-1),d}}\sqrt{\E[\|\nabla p(\bx)\|_2^{2}]}\sqrt{\sup_{\|\bx\|_2=1}\|\nabla p(\bx)\|_2^2},
\end{align*}
which implies that $\sup_{\|\bx\|_2=1}\|\nabla p(\bx)\|_2^{2}\le N_{2(t-1),d}\E[\|\nabla p(\bx)\|_2^{2}]\le t(d+2t-2)N_{2(t-1),d}\|p\|_2^2$.

Since $\left\|\frac{\partial^jp(\bx)}{\partial\bx^j}\right\|_2^2\le\left\|\frac{\partial^jp(\bx)}{\partial\bx^j}\right\|_F^2$,
it suffices to obtain an upper bound for $\sup_{\|\bx\|_2=1}\left\|\frac{\partial^jp(\bx)}{\partial\bx^j}\right\|_F^2$.
Noting that $\left\|\frac{\partial^jp(\bx)}{\partial\bx^j}\right\|_F^2\in\Omega^d_{2(t-j)}$, by~\cref{fact:infinity-to-l2}, we have that
\begin{align*}
\sup_{\|\bx\|_2=1}\left\|\frac{\partial^jp(\bx)}{\partial\bx^j}\right\|_F^2
&\le\sqrt{N_{2(t-j),d}}\sqrt{\E\left[\left\|\frac{\partial^jp(\bx)}{\partial\bx^j}\right\|_F^4\right]}\\
&\le\sqrt{N_{2(t-j),d}}\sqrt{\E\left[\left\|\frac{\partial^jp(\bx)}{\partial\bx^j}\right\|_F^2\right]}\sqrt{\sup_{\|\bx\|_2=1}\left\|\frac{\partial^jp(\bx)}{\partial\bx^j}\right\|_F^2},
\end{align*}
which implies that $\sup_{\bx\in\mathbb{S}^{d-1}}\left\|\frac{\partial^jp(\bx)}{\partial\bx^j}\right\|_F^2\le N_{2(t-j),d}\E\left[\left\|\frac{\partial^jp(\bx)}{\partial\bx^j}\right\|_F^2\right]$.
Noting that $\frac{\partial p(\bx)}{\partial x_i}\in\Omega^d_{t-1}$, by~\cref{lem:grad-bound}, we have that
\begin{align*}
&\quad\E\left[\left\|\frac{\partial^2p(\bx)}{\partial\bx^2}\right\|_F^2\right]=\E\left[\sum_{i_1,i_2\in[d]}\left(\frac{\partial^2p(\bx)}{\partial x_{i_1}\partial x_{i_2}}\right)^2\right]=\sum_{i_1=1}^d\E\left[\sum_{i_2=1}^d\left(\frac{\partial}{\partial x_{i_2}}\left(\frac{\partial p(\bx)}{\partial x_{i_1}}\right)\right)^2\right]\\&\le (t-1)(d+2t-4)\sum_{i_1=1}^d\E\left[\left(\frac{\partial p(\bx)}{\partial x_{i_1}}\right)^2\right]\le t(d+2t-2)\E[\|\nabla p(\bx)\|_2^2]\le t^2(d+2t-2)^2\|p\|_2^2 \;.
\end{align*}
In general, noting that $\frac{\partial^{j-1} p(\bx)}{\partial x_{i_1}\cdots\partial x_{i_{j-1}}}\in\Omega^d_{t-j+1}$, by~\cref{lem:grad-bound}, we have that
\begin{align*}
&\quad\E\left[\left\|\frac{\partial^jp(\bx)}{\partial\bx^j}\right\|_F^2\right]
=\E\left[\sum_{i_1,\ldots,i_j\in[d]}\left(\frac{\partial^2p(\bx)}{\partial x_{i_1}\ldots\partial x_{i_j}}\right)^2\right]
\\&=\sum_{i_1,\ldots,i_{j-1}\in[d]}\E\left[\sum_{i_j=1}^d\left(\frac{\partial}{\partial x_{i_j}}\left(\frac{\partial^{j-1} p(\bx)}{\partial x_{i_1}\ldots x_{i_{j-1}}}\right)\right)^2\right]
\\&\le(t-j+1)(d+2(t-j))\sum_{i_1,\ldots,i_{j-1}\in[d]}\E\left[\left(\frac{\partial^{j-1} p(\bx)}{\partial x_{i_1}\ldots x_{i_{j-1}}}\right)^2\right]
\\&\le t(d+2t-2)\E\left[\left\|\frac{\partial^{j-1}p(\bx)}{\partial\bx^{j-1}}\right\|_F^2\right]\le t^j(d+2t-2)^j\|p\|_2^2 \;.
\end{align*}
Therefore, we have that
\[ \sup_{\|\bx\|_2=1}\left\|\frac{\partial^jp(\bx)}{\partial\bx^j}\right\|_F^2\le N_{2(t-j),d}\E\left[\left\|\frac{\partial^jp(\bx)}{\partial\bx^j}\right\|_F^2\right]\le t^j(d+2t-2)^jN_{2(t-j),d}\|p\|_2^2 \;.\] 
This completes the proof.
\end{proof}

\subsection{Proof of~\cref{lem:poly-concentration}}\label{ssec:poly-concentration}

Let $p_1,\ldots,p_N\in\Omega$ be an orthonormal basis,
i.e., $\E[p_i(\bx)p_j(\bx)]=\I[i=j]$.
Let vector $\bp(\bx)\eqdef[p_1(\bx),\ldots,p_N(\bx)]$.
We have that $\E[\bp(\bx)]=0$ and $\mathbf{Cov}[\bp(\bx)]=\bI_N$.
\begin{align*}
\pr\left[\left\|\frac{1}{r}\sum_{i=1}^r\bp(\bx_i)\right\|_2\ge\eta\right]
&=\pr\left[\frac{1}{r^2}\left\|\sum_{i=1}^r\bp(\bx_i)\right\|_2^2\ge\eta^2\right]
=\pr\left[\frac{1}{r^2}\sum_{j=1}^{N}\left(\sum_{i=1}^rp_j(\bx_i)\right)^2\ge\eta^2\right]
\\&\le\frac{1}{\eta^2r^2}\sum_{j=1}^{N}\E\left[\left(\sum_{i=1}^rp_j(\bx_i)\right)^2\right]=\frac{N}{r\eta^2}.
\end{align*}
We now assume that $\frac{1}{r}\left\|\sum_{i=1}^r\bp(\bx_i)\right\|_2\le\eta$.
Let $p\in\Omega$ be an arbitrary polynomial. We can write $p(\bx)=\sum_{j=1}^N\alpha_jp_j(\bx)$, where $\|p\|_2^2=\sum_{j=1}^N \alpha_j^2$. We have that
\begin{align*}
&\quad\frac{1}{r}\left|\sum_{i=1}^rp(\bx_i)\right|=\frac{1}{r}\left|\sum_{i=1}^r\sum_{j=1}^{N}\alpha_jp_j(\bx_i)\right|\le\frac{1}{r}\sum_{j=1}^{N}|\alpha_j|\left|\sum_{i=1}^rp_j(\bx_i)\right|
\\&\le\frac{1}{r}\sqrt{\sum_{j=1}^{N}\alpha_j^2}\sqrt{\sum_{j=1}^{N}\left(\sum\nolimits_{i=1}^rp_j(\bx_i)\right)^2}
=\frac{\|p\|_2}{r}\left\|\sum\nolimits_{i=1}^r\bp(\bx_i)\right\|_2
\le\eta\|p\|_2,
\end{align*}
where the second inequality follows from Cauchy-Schwarz.
This completes the proof.

\subsection{Proof of~\cref{lem:increment}}\label{ssec:increment}

By definition of $\nabla_o p(\by)$, we have that
\begin{align*}
&\quad p(\bz)-p(\by)
=\frac{p(\by+\delta\cdot\nabla_o p(\by))}{\|\by+\delta\cdot\nabla_o p(\by)\|_2^t}-p(\by)
\\&=\frac{p(\by+\delta\cdot\nabla_o p(\by))-p(\by)}{(1+\delta^2\|\nabla_op(\by)\|_2^2)^{t/2}}-\left(1-\frac{1}{(1+\delta^2\|\nabla_op(\by)\|_2^2)^{t/2}}\right)p(\by)\\
&\ge\frac{p(\by+\delta\cdot\nabla_o p(\by))-p(\by)}{(1+\delta^2\|\nabla_op(\by)\|_2^2)^{t/2}}-\left(1-\exp(-t\delta^2\|\nabla_op(\by)\|_2^2/2)\right)|p(\by)|\\&\ge\frac{p(\by+\delta\cdot\nabla_o p(\by))-p(\by)}{(1+\delta^2\|\nabla_op(\by)\|_2^2)^{t/2}}-t\delta^2\|\nabla_op(\by)\|_2^2|p(\by)|/2 \;.
\end{align*}
We bound $p(\by+\delta\cdot\nabla_o p(\by))-p(\by)$ as follows:
Let $f(s)=p(\by+s\bv)$ for some unit vector $v\in\R^d$. 
Noting that $p$ is a degree-$t$ homogeneous polynomial, by Taylor expansion, we have that $f(s)=f(0)+\sum_{j=1}^t \frac{f^{(j)}(0)s^j}{j!}$.
By elementary calculation, we have that $f'(0)=\bv^{\intercal}\nabla p(\by),f''(0)=\bv^{\intercal}\frac{\partial^2p(\by)}{\partial\by^2}\bv,\ldots,f^{(t)}(0)=\left\langle\bv^{\otimes t},\frac{\partial^t p(\by)}{\partial\by^t}\right\rangle$. By taking $\bv$ to be the direction of $\nabla_o p(\by)$, i.e., $\bv=\frac{\nabla_op(\by)}{\|\nabla_op(\by)\|_2}$, we have that
\begin{align*}
p(\by+\delta\cdot\nabla_op(\by))-p(\by)=f(\delta\|\nabla_op(\by)\|_2)-f(0)=\sum_{j=1}^t\frac{\left\langle\nabla_op(\by)^{\otimes j},\frac{\partial^j p(\by)}{\partial\by^j}\right\rangle\delta^j}{j!} \;.
\end{align*}
Noting that the first order term is $\delta\|\nabla_o p(\by)\|_2^2$, 
it suffices to show that the absolute value of 
$\sum_{j=2}^t\frac{\left\langle\nabla_op(\by)^{\otimes j},\frac{\partial^j p(\by)}{\partial\by^j}\right\rangle\delta^j}{j!}$ is sufficiently small.
Applying~\cref{lem:higher-order-max} yields
\begin{align*}
&\quad\left|\sum_{j=2}^t\frac{\left\langle\nabla_op(\by)^{\otimes j},\nabla^jp(\by)\right\rangle\delta^j}{j!}\right|
\le\sum_{j=2}^t\frac{\delta^j\|\nabla_op(\by)\|_2^j\left\|\frac{\partial^jp(\by)}{\partial\by^j}\right\|_2}{j!}
\\&=\delta\|\nabla_o p(\by)\|_2^2\sum_{j=2}^t\frac{\delta^{j-1}\|\nabla_op(\by)\|_2^{j-2}\left\|\frac{\partial^jp(\by)}{\partial\by^j}\right\|_2}{j!}
\\&\le\delta\|\nabla_o p(\by)\|_2^2\left(\sum_{j=2}^t\frac{\delta^{j-1}\|\nabla p(\by)\|^{2j-4}_2}{2j!}+\sum_{j=2}^t\frac{\delta^{j-1}\left\|\frac{\partial^jp(\by)}{\partial\by^j}\right\|_2^2}{2j!}\right)
\\&\le\delta\|\nabla_o p(\by)\|_2^2\left(\sum_{j=2}^t\frac{\delta^{j-1}(t(d+2t-2)N_{2(t-1),d})^{j-2}}{2j!}+\sum_{j=2}^t\frac{\delta^{j-1} t^j(d+2t-2)^jN_{2(t-j),d}}{2j!}\right) \;.
\end{align*}
Therefore, we will have that 
$p(\by+\delta\cdot\nabla_op(\by))-p(\by)\ge C' \; \delta\|\nabla_o p(\by)\|_2^2$ 
for some universal constant $0<C'<1$, as long as $\delta\le1/N_{2t,d}^2$. 
Thus, by~\cref{lem:grad-bound}, we have that
\begin{align*}
p(\bz)-p(\by)&\ge\frac{p(\by+\delta\cdot\nabla_o p(\by))-p(\by)}{(1+\delta^2\|\nabla_op(\by)\|_2^2)^{t/2}}-t\delta^2\|\nabla_op(\by)\|_2^2|p(\by)|/2
\\&=\frac{C'\delta\|\nabla_op(\by)\|_2^2}{(1+\delta^2\|\nabla_op(\by)\|_2^2)^{t/2}}-t\delta^2\|\nabla_op(\by)\|_2^2|p(\by)|/2
\\&=C'\delta\|\nabla_o p(\by)\|_2^2\exp(-t\delta^2\|\nabla_o p(\by)\|_2^2/2)-t\delta^2\|\nabla_op(\by)\|_2^2|p(\by)|/2
\\&\ge C'\delta\|\nabla_o p(\by)\|_2^2\left(1-t\delta^2\|\nabla p(\by)\|_2^2/2-t\delta|p(\by)|/2C'\right)
\\&\ge\delta\|\nabla_o p(\by)\|_2^2\left(C'(1-t^2\delta^2(d+2t-2)N_{2(t-1),d}/2)-t\delta\sqrt{N_{t,d}}/2\right)
\\&\ge C\delta\|\nabla_o p(\by)\|_2^2 \;,
\end{align*}
for some universal constant $0<C<1$, as long as $\delta\le1/N_{2t,d}^2$.
This completes the proof.

\subsection{Proof of~\cref{thm:sph-construction}}\label{ssec:sph-construction}
Let $p\in\partial\Omega_t^d$.
Since 
\[ \|\nabla p(\by)\|_2^2=\|\nabla_o p(\by)\|_2^2+\langle\by,\nabla p(\by)\rangle^2
=\|\nabla_o p(\by)\|_2^2+t^2p(\by)^2 \;, \]
by~\cref{lem:increment},
we have that $p(\bz)\ge p(\by)+C\delta(\|\nabla p(\by)\|_2^2-t^2p(\by)^2)$. 
Let \[ q(\by)=p(\by)+C\delta\|\nabla_o p(\by)\|_2^2=p(\by)+C\delta(\|\nabla p(\by)\|_2^2-t^2p(\by)^2) \;.\]
By definition, we have that 
$q(\by)-\E[q(\by)]=p(\by)+C\delta(\|\nabla p(\by)\|_2^2-t^2p(\by)^2)-C\delta\E[\|\nabla_o p(\by)\|_2^2]$,
which is a polynomial of degree at most $2t$ and contains only monomials of degree $2t,2t-2,t,0$.
Let $\Omega$ be the subspace of polynomials in $d$-variables containing all monomials of degree $2t,2t-2,t,0$.
In this way, the dimension of $\Omega$ is 
\[ N=\binom{d+2t-1}{d-1}+\binom{d+2t-3}{d-1}+\binom{d+t-1}{d-1}+1\le3N_{2t,d} \;. \]
Applying~\cref{lem:poly-concentration} yields that with probability at least $1-\frac{N}{r\eta^2}$,
we have that $\left|\frac{1}{r}\sum_{i=1}^rq(\by_i)-\E[q(\by)]\right|\le\eta\|q(\by)-\E[q(\by)]\|_2,\forall q\in\Omega$.
Therefore, we have that
\begin{align*}
\frac{1}{r}\sum_{i=1}^rp(\bz_i)\ge\frac{1}{r}\sum_{i=1}^rq(\by_i)\ge\E[q(\by)]-\eta\|q(\by)-\E[q(\by)]\|_2=\E[q(\by)]-\eta\sqrt{\E[q(\by)^2]-\E[q(\by)]^2}.
\end{align*}
By elementary calculation, we have that
\begin{align*}
&\quad\E[q(\by)^2]-\E[q(\by)]^2=\E[(p(\by)+C\delta\|\nabla_o p(\by)\|_2^2)^2]-C^2\delta^2\E[\|\nabla_o p(\by)\|_2^2]^2\\&=\E[p(\by)^2]+2C\delta\E[p(\by)\|\nabla_o p(\by)\|_2^2]+C^2\delta^2\E[\|\nabla_o p(\by)\|_2^4]-C^2\delta^2\E[\|\nabla_o p(\by)\|_2^2]^2\\&=\E[p(\by)^2]+C^2\delta^2\E[\|\nabla_o p(\by)\|_2^4]-C^2\delta^2\E[\|\nabla_o p(\by)\|_2^2]^2\\&=1+C^2\delta^2\E[\|\nabla_o p(\by)\|_2^4]-C^2\delta^2\E[\|\nabla_o p(\by)\|_2^2]^2,
\end{align*}
where the second equality is due to $p(\by)$ being odd and 
\begin{align*}
\|\nabla_o p(-\by)\|_2^2&=\|\nabla p(-\by)\|_2^2-t^2p(-\by)^2 
=\|\nabla(-p(\by))\|_2^2-t^2(-p(\by))^2\\&=\|\nabla p(\by)\|_2^2-t^2p(\by)^2=\|\nabla_o p(\by)\|_2^2 \;.
\end{align*}
By~\cref{lem:grad-bound} and~\cref{lem:higher-order-max}, we have that
$\E[\|\nabla_o p(\by)\|_2^2]\ge d-1$ and 
$$\E[\|\nabla_o p(\by)\|_2^4]\le\E[\|\nabla p(\by)\|_2^2]\sup_{\|\by\|_2=1}\|\nabla p(\by)\|_2^2\le t^2(d+2t-2)^2N_{2(t-1),d}.$$
Therefore, we have that
\begin{align*}
\frac{1}{r}\sum_{i=1}^rp(\bz_i)&\ge \E[q(\by)]-\eta\sqrt{\E[q(\by)^2]-\E[q(\by)]^2}\\
&\ge C\delta\E[\|\nabla_o p(\by)\|_2^2]-\eta\sqrt{1+C^2\delta^2\E[\|\nabla_o p(\by)\|_2^4]-C^2\delta^2\E[\|\nabla_o p(\by)\|_2^2]^2}\\
&\ge C\delta(d-1)-\eta\sqrt{1+C^2\delta^2(t^2(d+2t-2)^2N_{2(t-1),d}-(d-1)^2)}\\
&=C \delta\left(d-1-\eta\sqrt{\frac{1}{C^2\delta^2}+t^2(d+2t-2)^2N_{2(t-1),d}-(d-1)^2}\right) \;.
\end{align*}
Taking $\delta=1/N_{2t,d}^2$ and $\eta=\frac{Cd}{3N_{2t,d}^2}$ yields that with probability at least 
\[ 1-\frac{N}{r\eta^2}\ge1-\frac{27}{C^2d^2}\ge99/100 \;, \]
\begin{align*}
\frac{1}{r}\sum_{i=1}^rp(\bz_i)&
\ge C\delta\left(d-1-\eta\sqrt{N_{2t,d}^4/C^2+t^2(d+2t-2)^2N_{2(t-1),d}-(d-1)^2}\right)\\
&> C\delta\left(d/2-\eta\sqrt{2N_{2t,d}^4/C^2}\right)\ge0 \;.
\end{align*}

\end{document}